\documentclass{article}

\usepackage[final]{neurips_2022}

\usepackage[utf8]{inputenc}
\usepackage{pdflscape}
\usepackage[final, babel]{microtype}
\usepackage{algorithmicx} 
\usepackage[T1]{fontenc}
\usepackage[noend]{algpseudocode}
\usepackage{amsthm, amsmath, amsfonts, amssymb, graphicx}
\usepackage{hyperref}
\usepackage{color}
\usepackage{mathrsfs}
\usepackage{enumitem}
\usepackage{cancel}
\usepackage{bm}
\usepackage{multirow}
\usepackage{booktabs}
\usepackage{makecell}
\usepackage{caption}
\usepackage{thmtools}
\usepackage{thm-restate}
\usepackage{hhline}
\usepackage{cite}
\usepackage{rotating} 
\usepackage{natbib}
\usepackage{times}
\usepackage[ruled,linesnumbered]{algorithm2e}
\usepackage{xcolor}         
\usepackage{subfigure}
\usepackage{graphicx}

\renewcommand{\epsilon}{\varepsilon}

\newcommand{\trans}{^{\top}}

\newcommand{\cA}{\mathcal{A}}
\newcommand{\cS}{\mathcal{S}}
\newcommand{\cE}{\mathcal{E}}
\newcommand{\cZ}{\mathcal{Z}}

\newcommand{\EE}{\mathbb{E}}

\newcommand{\hV}{{\hat{V}}}
\newcommand{\hQ}{{\hat{Q}}}

\newcommand{\hP}{{\hat{P}}}
\newcommand{\ph}{{h'}}
\newcommand{\pk}{{k'}}
\newcommand{\tk}{{\tilde{k}}}
\newcommand{\nk}{{k}}

\newcommand{\mP}{{\mathbb{P}}}
\newcommand{\mR}{{\mathbb{R}}}

\newcommand{\hmP}{{\hat{\mathbb{P}}}}

\newcommand\numberthis{\addtocounter{equation}{1}\tag{\theequation}}

\let\hat\widehat
\let\tilde\widetilde

\newtheorem{theorem}{Theorem}[section]
\newtheorem{lemma}[theorem]{Lemma}
\newtheorem{corollary}[theorem]{Corollary}
\newtheorem{remark}[theorem]{Remark}

\newtheorem{definition}[theorem]{Definition}

\newtheorem{condition}[theorem]{Condition}
\newtheorem{assumption}{Assumption}

\newcount\Comments  
\newcommand{\kibitz}[2]{\ifnum\Comments=1\textcolor{#1}{#2}\fi}

\newcommand{\blue}[1]{{\color{blue} #1}}
\newcommand{\red}[1]{{\color{red} #1}}

\DeclareMathOperator*{\argmax}{arg\,max}

\newcommand{\algO}{\text{Alg}^\text{O}}
\newcommand{\algE}{\text{Alg}^\text{E}}
\newcommand{\alg}{\text{Alg}}
\newcommand{\envO}{\text{Env}^\text{O}}
\newcommand{\envE}{\text{Env}^\text{E}}

\newcommand{\piO}{{\pi}^\text{O}}
\newcommand{\piE}{{\pi}^\text{E}}
\newcommand{\piPVI}{{\pi}^\text{PVI}}
\newcommand{\groupO}{\text{G}^\text{O}}
\newcommand{\groupE}{\text{G}^\text{E}}
\newcommand{\tauO}{{\tau}^\text{O}}
\newcommand{\tauE}{{\tau}^\text{E}}

\newcommand{\surplus}{\mathbf{E}}
\newcommand{\Regret}{\text{Regret}}
\newcommand{\epsClip}{\epsilon_\text{Clip}}

\newcommand{\Clip}{\text{Clip}}

\newcommand{\lc}{\lceil}
\newcommand{\rc}{\rceil}

\newcommand{\cEB}{\cE_{\textbf{Bonus}}}
\newcommand{\cECk}{\cE_{\textbf{Con},k}}
\newcommand{\cEBk}{\cE_{\textbf{Bonus},k}}
\newcommand{\cEOk}{\cE_{\algO,k}}

\newcommand{\btau}{\bar{\tau}}
\newcommand{\ttau}{\tilde{\tau}}

\newcommand{\tOmega}{\tilde{\Omega}}
\newcommand{\tTheta}{\tilde{\Theta}}
\newcommand{\tO}{\tilde{O}}

\newcommand{\scalar}{\lambda}

\newcommand{\Gl}{G^{\text{lower}}}
\newcommand{\Gu}{G^{\text{upper}}}

\newcommand{\compilehidecomments}{true}
\ifthenelse{ \equal{\compilehidecomments}{true} }{%
	\newcommand{\wei}[1]{}
	\newcommand{\jiawei}[1]{}
	\newcommand{\nan}[1]{}
	\newcommand{\li}[1]{}
}{
	\newcommand{\wei}[1]{{\color{red}  [\text{Wei:} #1]}}
	\newcommand{\jiawei}[1]{{\color{orange} [\text{Jiawei:} #1]}}
	\newcommand{\nan}[1]{{\color{teal} [\text{Nan:} #1]}}
	\newcommand{\li}[1]{{\color{magenta} [\text{Li:} #1]}}
}

\newcommand{\compilefullversion}{true} 
\ifthenelse{\equal{\compilefullversion}{false}}{%
	\newcommand{\OnlyInFull}[1]{}
	\newcommand{\OnlyInShort}[1]{#1}
}{%
	\newcommand{\OnlyInFull}[1]{#1}%
	\newcommand{\OnlyInShort}[1]{}%
}%

\title{Tiered Reinforcement Learning: Pessimism in the Face of Uncertainty and Constant Regret}

\author{
 Jiawei Huang$^1$\thanks{Work done during the internship at Microsoft Research Asia.} 
 \quad
 Li Zhao$^2$ 
 \quad
 Tao Qin$^2$ 
 \quad
 Wei Chen$^2$ 
 \quad
 Nan Jiang$^1$
 \quad
 Tie-Yan Liu$^2$ \\
 $^1$ Department of Computer Science, University of Illinois at Urbana-Champaign\\
 \texttt{\{jiaweih, nanjiang\}@illinois.edu}\\
 $^2$ Microsoft Research Asia\\
 \texttt{\{lizo, taoqin, weic, tyliu\}@microsoft.com}
}

\begin{document}
\maketitle

\begin{abstract}
We propose a new learning framework that captures the tiered structure of many real-world user-interaction applications, where the users can be divided into two groups based on their different tolerance on exploration risks and should be treated separately.
In this setting, we simultaneously maintain two policies $\pi^{\text{O}}$ and $\pi^{\text{E}}$: $\pi^{\text{O}}$ (``O'' for ``online'') interacts with more risk-tolerant users from the first tier and minimizes regret by balancing exploration and exploitation as usual, 
while $\pi^{\text{E}}$ (``E'' for ``exploit'') exclusively focuses on exploitation for risk-averse users from the second tier utilizing the data collected so far.
An important question is whether such a separation yields advantages over the standard online setting (i.e., $\pi^{\text{E}}=\pi^{\text{O}}$) for the risk-averse users. 
We individually consider the gap-independent vs.~gap-dependent settings. For the former, we prove that the separation is indeed not beneficial from a minimax perspective.
For the latter, we show that if choosing Pessimistic Value Iteration as the exploitation algorithm to produce $\pi^{\text{E}}$, we can achieve a constant regret for risk-averse users independent of the number of episodes $K$, which is in sharp contrast to the $\Omega(\log K)$ regret for any online RL algorithms in the same setting, while the regret of $\pi^{\text{O}}$ (almost) maintains its online regret optimality and does not need to compromise for the success of $\pi^{\text{E}}$.
\end{abstract}

\allowdisplaybreaks
\section{Introduction}\label{sec:introduction}
Reinforcement learning (RL) has been applied to many real-world user-interaction applications to provide users with better services, such as in recommendation systems \citep{afsar2021reinforcement} and medical treatment \citep{yu2021reinforcement, lipsky2001idea}.
In those scenarios, the users take the role of the environments and the interaction strategies (e.g. recommendation or medical treatment) correspond to the agents in RL. 
In the theoretical study of such problems, 
most of the existing literature adopts the online interaction protocol, where in each episode $k\in[K]$, the learning agent executes a policy $\pi_k$ to interact with users (i.e. environments), receives new data to update the policy, and moves on to the next episode. 
While this formulation \emph{treats each user equivalently} when optimizing the regret,
many scenarios have a special ``\textbf{Tiered Structure}''\footnote{We consider the cases with two tiers in this paper.}: \emph{users can be divided into multiple groups depending on their different preference and tolerance about the risk that results from the necessary exploration to improve the policy}, and such grouping is available to the learner in advance so it would be better to treat them separately. 
As a concrete example, in medical treatment, after a new treatment plan comes out, some courageous patients or paid volunteers (denoted as $\groupO$; ``O'' for ``Online'') may prefer it given the potential risks, while some conservative patients (denoted as $\groupE$; ``E'' for ``Exploit'') may tend to receive mature and well-tested plans, even if the new one is promising to be more effective. 
As another example, companies offering recommendation services may recruit paid testers or use bonus to attract customers ($\groupO$) to interact with the system to shoulder the majority of the exploration risk during policy improvement, which may result in better service (low regret) for the remaining customers ($\groupE$).
Moreover, many online platforms have free service open for everyone ($\groupO$), while some users are willing to pay for enhanced service ($\groupE$). 
If we follow the traditional online setting and treat the users in these two groups equivalently, then in expectation each group will suffer the same regret and risk.
In contrast, if we leverage the group information by using policies with different risk levels to interact with different groups, it is potentially possible to transfer some exploration risks from users in $\groupE$ to $\groupO$, while the additional risks suffered by $\groupO$ will be compensated in other forms (such as payment, the users' inherent motivation, or the free service itself).
\begin{figure}
    \centering
    \includegraphics[scale=0.38]{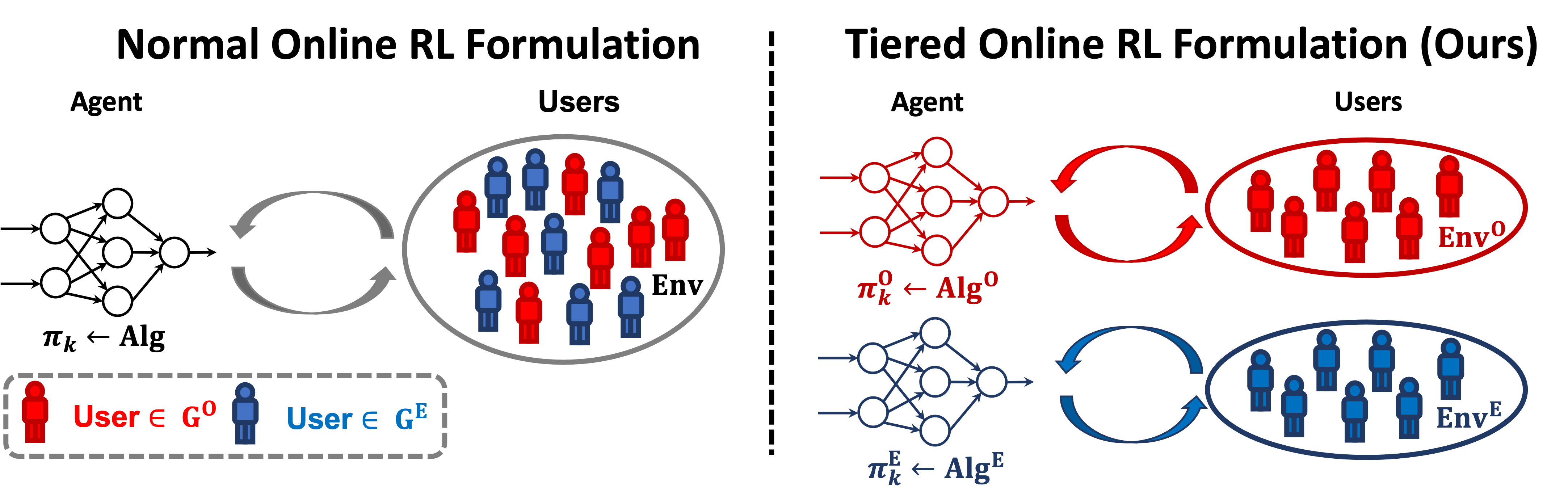}
    \caption{Comparison between the standard setting and our tiered RL setting (\#Tiers = 2), where we use red and blue to color users from different groups. The main difference is that, in the standard setting (LHS), the learner does not distinguish users from different groups and treats them equivalently with a single policy $\pi_k$ produced by algorithm $\alg$, while in our setting (RHS), we leverage the tier information and interact with different groups with different policies $\piE$ and $\piO$.}
    \label{fig:tieredRL_formulation}
\end{figure}

To make our objective more clear, we abstract the problem setting into Frw.~\ref{alg:general_learning_framework} and compare it with the standard online setting in Fig. \ref{fig:tieredRL_formulation}, where we use $\algO$ and $\algE$ to denote the two algorithms producing policies $\piE$ and $\piO$ to interact with users in $\groupO$ and $\groupE$, respectively. 
To enable theoretical analyses, we do adopt a few simplification assumptions while still modelling the core challenges in the aforementioned scenarios: 
firstly, at each iteration of Frw.~\ref{alg:general_learning_framework}, the algorithms will interact with and collect one trajectory from each group, 
which assumes that users from two groups will come to seek for service in pair with the same frequency.
In practice, usually, the users come in random order and the frequencies from different groups are not the same; see Appx.~\ref{appx:discussion_abs_framework_examples} for why our abstraction is still a valid surrogate and how our results can be generalized.
Secondly, for convenience, 
\emph{we only use the samples generated from $\groupO$}, because $\algE$ is expected to best exploit the available information and not encouraged to perform intelligent exploration. 
Nonetheless, our results hold with minor modifications if one also chooses to use trajectories from $\groupE$. 
Thirdly, we assume that 
the dynamics and rewards during the interactions with users in different groups are all the same (i.e. $\text{Env}=\text{Env}^\text{O}=\text{Env}^\text{E}$). It is possible that the users in different tiers can behave differently, and we leave the relaxation of such an assumption to future work.
\begin{algorithm}[H]
    \SetAlgorithmName{Framework} \\
    \textbf{Input}: $\envO$ and $\envE$ // \blue{Note that $\envO=\envE$} \\
    Initialize $D_1 \gets \{\}.$ \\
    \For{$k = 1,2,...,K$}{
        $\piO_k \gets \algO(D_k)$; $\piE_k \gets \algE(D_k)$.\\
        $\piO_k$ interacts with customers/users/patients in $\groupO$ (i.e. $\envO$), and collect data $\tauO_k$. \\
        $\piE_k$ interacts with customers/users/patients in $\groupE$ (i.e. $\envE$), and collect data $\tauE_k$. \\
        $D_{k+1} = D_k \cup \{\tauO_k\}$. // \blue{We do not consider to use $\tauE_k$ in this paper.}
    }
    \caption{The Tiered RL Framework}\label{alg:general_learning_framework}
\end{algorithm}

Similar to the online setting, we use the expected pseudo-regret to measure the performance of the algorithm, which is formalized in Def.~\ref{def:Pseduo_regret}. 
The key problem we would like to investigate
is provable benefits of leveraging the tiered structure by Frw.~\ref{alg:general_learning_framework} comparing with the standard online setting:

\begin{center}
    \textbf{Is it possible for $\Regret(\algE)$ to be strictly lower than any online learning algorithms in certain scenarios, while keeping $\Regret(\algO)$ near-optimal?}
\end{center}
Note that we still expect regret of $\algO$ to enjoy near-optimal regret guarantees, which is a reasonable requirement as the experience of users in $\groupO$ also matters in many of our motivating applications. 
We regard the above problem formulation as \textbf{our first contribution}, which is mainly conceptual.

As \textbf{our second contribution},  Sec.~\ref{sec:Regret_lower_bound_without_gap_assump} shows that $\algE$ has the same minimax gap-independent lower bound as online learning algorithms.
This result reveals the difficulty to leverage tiered structure in standard tabular MDPs, and motivates us to investigate the benefits under the gap-dependent setting,
which is frequently considered in the Multi-Armed Bandit (MAB) \citep{lattimore2020bandit, rouyer2020tsallis} and RL literature \citep{xu2021fine,simchowitz2019non}.

As \textbf{our third contribution} and our main technical results, Sec.~\ref{sec:PVI_and_ConstantRegret} establishes provable benefits of Frw.~\ref{alg:general_learning_framework} by proposing a new algorithmic framework and showing $\Regret(\algE)$ is constant and independent of the number of episodes $K$, which is in sharp contrast with the $\Omega(\log K)$ regret lower bound for any algorithms in the standard online setting that do not leverage the tiered structure. Specifically, we use Pessimistic Value Iteration (PVI) as $\algE$ for exploitation to interact with $\groupE$, while $\algO$ can be arbitrary online algorithms with near-optimal regret.
Concretely, we first study stochastic MABs as a warm-up, where we choose $\algE$ to be LCB (Lower Confidence Bonus), a degenerated version of PVI in bandits, and choose UCB (Upper Confidence Bonus) as $\algO$ for a concrete case study.
We prove that $\algE$ can achieve constant pseudo-regret $\tilde{O}(\sum_{\Delta_i > 0}  (A-i)\Big(\frac{1}{\Delta_i}- \frac{\Delta_i}{\Delta_{i-1}^2}\Big))$ with $A$ being the number of actions and $\Delta_i$'s being the gaps with $\Delta_1 \geq...\geq\Delta_{A-1}\geq\Delta_A = 0$, 
while $\algO$ is near-optimal due to the regret guarantee of UCB.
After that, Sec.~\ref{sec:Tabular_RL}  extends the success of PVI to tabular MDPs, and establishes results that apply to a wide range of online algorithms $\algO$ with near-optimal regret.
Although the benefits of pessimism have been widely recognized in offline RL \citep{jin2021pessimism}, to our knowledge, we are the first to study PVI in a gap-dependent online setting. We also contribute several novel techniques for overcoming the difficulties in achieving constant regret, and defer their summary to Sec.~\ref{sec:Tabular_RL}.
Moreover, in Appx.\ref{appx:experiments}, we report experiment results to demostrate the advantage of leveraging tiered structure as predicted by theory.

\textbf{Closest Related Work}\label{sec:related_work}
Due to space limit, we only discuss the closest related work here and defer the rest to Appx.~\ref{appx:detailed_related_work}.
To our knowledge, there is no previous works on leveraging tiered structure in MDPs. In the bandit setting, there is a line of related works studying decoupling exploration and exploitation \citep{avner2012decoupling,rouyer2020tsallis}, where \citep{rouyer2020tsallis} studied ``best of both worlds'' methods and reported a similar constant regret. 
First, in stochastic bandits, there are many cases when our result is tighter than theirs,
(see a detailed comparison in Sec.~\ref{sec:analysis_for_MAB}), 
and more importantly, our methods can naturally extend to RL (i.e., MDPs), whereas a similar extension of their techniques can run into serious difficulties:
they relied on \textit{importance sampling} to provide unbiased off-policy estimation for policy value, which incurs the infamous ``curse of horizon'', a.k.a., a sample complexity \textit{exponential} in the planning horizon $H$ in long-horizon RL (see examples in Sec.~2 in \citep{liu2018breaking}). Our approach overcomes this difficulty by developing a pessimism-based learning framework, which is fundamentally different from their approach and requires several novel techniques in the analyses. 
Second, they did not provide any guarantee for the regret of exploration algorithm, whereas in our results the regret of $\algO$ can be near-optimal, which we believe is also important as the experience of users in $\groupO$ also matters in many of our motivating applications. Third, their bandit results require a unique best arm, whereas we allow the optimal arms/policies to be non-unique, which can cause non-trivial difficulties in the analyses as we will discuss in Sec.~\ref{sec:tabular_RL_main_analysis}.

%
%
%
%

\section{Preliminary and Problem formulation}\label{sec:learning_framework}
\textbf{Stochastic Multi-Armed Bandits (MABs)}~~The MAB model consists of a set of arms $\cA=\{1,2,...,A\}$. When sampling an arm $i\in\cA$, the agent observes a random variable $r_i\in[0,1]$. We use 
$\mu_i = \EE[r_i]$ to denote the mean value for arm $i$ for each $i\in\cA$.
We allow the optimal arms to be non-unique.
For simplicity of notation, we assume the arms are ordered such that $\mu_1 \leq \mu_2...\leq \mu_A$.

\textbf{Finite-Horizon Tabular Markov Decision Processes (MDPs)}~~For the reinforcement learning (RL) setting, we consider the episodic tabular MDPs denoted by $M(\cS, \cA, H, P, r)$, where $\cS$ is the finite state space, $\cA$ is the finite action space, $H$ is the horizon length, and $\mP=\{\mP_h\}_{h=1}^H$ and $r=\{r_h\}_{h=1}^H$ are the time-dependent transition and reward functions, respectively. 
We assume all steps share the state and action space (i.e. $\cS_1=\cS_2...=\cS_H=\cS$, $\cA_1=\cA_2...=\cA_H=\cA$) while the transition and reward functions can be different.
At the beginning of each episode, the environment will start from a fixed initial state $s_1$ (w.l.o.g.). Then, for each time step $h\in[H]$, the agent selects an action $a_h \in \cA$ based on the current state $s_h$, receives the reward $r_h(s_h,a_h)$, and observes the system transition to the next state $s_{h+1}$, until $s_{H+1}$ is reached. W.l.o.g., we assume the reward function $r$ is deterministic and our results can be easily extended to handle stochastic rewards.

A time-dependent policy is specified as $\pi=\{\pi_1,\pi_2,...,\pi_H\}$ with $\pi_h: \cS\to \Delta(\cA)$ for all $h\in[H]$. Here $\Delta(\cA)$ denotes the probability simplex over the action space.
With a slight abuse of notation, when $\pi_h$ is a deterministic policy, we use $\pi_h:\cS\to \cA$ to refer to a deterministic mapping. $V^\pi_h(s)$ and $Q^\pi_h(s,a)$ denote the value function and Q-function at step $h\in[H]$, which are defined as:
    $
    V^\pi_h(s)=\EE[\sum_{h'=h}^H r_{h'}(s_{h'},a_{h'})|s_h=s,\pi],\quad Q^\pi_h(s,a)=\EE[\sum_{h'=h}^H r_{h'}(s_{h'},a_{h'})|s_h=s,a_h=a,\pi].
    $

We use $V^{*}_h(\cdot):=\max_\pi V^\pi_h(\cdot)$ and $Q^{*}_h(\cdot,\cdot)=\max_\pi Q^{\pi}_h(\cdot,\cdot)$ to refer to the optimal state/action-value functions, and  $\Pi^*(s_h):=\{a_h|Q^{*}(s_h,a_h)=V^*_h(s_h)\}$ to denote the collection of all optimal actions at state $s_h$. With an abuse of notation, we define $\Pi^* := \{\pi: V_1^\pi(s_1) = V_1^*(s_1)\}$, i.e., the set of policies that maximize the total expected return. In this paper, when we say that the MDP has ``unique optimal (deterministic) policy'', it is up to the occupancy measure, that is, all policies in $\Pi^*$ share the same state-action occupancy $d^{\pi}(s_h,a_h):=\Pr(S_h=s_h,A_h=a_h|S_1=s_1,\pi)$ for all $h\in[H],s_h\in\cS_h,a_h\in\cA_h$.
In the following, we use $|\Pi^*|=1$ to refer to the case of unique optimal (deterministic) policy, where the cardinality of $\Pi^*$ is counted up to the equivalence of occupancies.
Besides, for any function $V:\cS\rightarrow\mR$, we denote $\mP_h V(s_h,a_h) := \EE_{s_{h+1}\sim \mP_h(\cdot|s_h,a_h)}[V(s_{h+1})]$.


\textbf{Gap-Dependent Setting}~~We follow the standard formulation of gap-dependent setting in previous bandits \citep{lattimore2020bandit} and RL literature \citep{simchowitz2019non, xu2021fine, dann2021beyond}. 
In bandits, the gap w.r.t.~arm $i$ is defined as $\Delta_i:=\max_{j\in[A]}\mu_j - \mu_i,\forall i \in [A]$, and we assume that there exists a strictly positive value $\Delta_{\min}$ such that, either $\Delta_i = 0$ or $\Delta_i \geq \Delta_{\min}$.
For tabular RL setting, we define $\Delta_h(s_h,a_h):=V^*(s_h)-Q^*(s_h,a_h),\forall h\in[H],s_h\in\cS_h,a_h\in\cA_h$. We use the same notation $\Delta_{\min}$ to refer to the minimal gap in tabular setting and assume that either $\Delta_h(s_h,a_h)=0$ or $\Delta_h(s_h,a_h) \geq \Delta_{\min}$. 

\textbf{Performance Measure}~~We use 
Pseudo-Regret defined below to measure the performance of $\algO$ and $\algE$. 
In the following, we will also use ``exploitation regret'' to refer to $\Regret_K(\algE)$.
\begin{definition}[Pseudo-Regret]\label{def:Pseduo_regret}
We define the regret of $\algO$ and $\algE$ to be:
\begin{align*}
    \Regret_K(\algO):=&\EE\left[\sum_{k=1}^K V_1^*(s_1)-V_1^{\piO_k}(s_1)\right];\quad 
    \Regret_K(\algE):=\EE\left[\sum_{k=1}^K V_1^*(s_1)-V_1^{\piE_k}(s_1)\right],
\end{align*}
where $\piO_k$ and $\piE_k$ are generated according to the procedure in Framework \ref{alg:general_learning_framework} and the expectation is taken over the randomness in data generation and algorithms.
\end{definition}

\section{Lower Bound of $\Regret(\algE)$ without Gap Assumption}\label{sec:Regret_lower_bound_without_gap_assump}
In this section, we show that, 
in normal tabular RL setting, for arbitrary algorithm pair $(\algO, \algE)$, even if we do not constrain $\algO$ to be near-optimal, the regret of $\algE$ has the same minimax lower bound as algorithms in online setting. We defer the formal statement and proof to Appendix~\ref{appx:lower_bound_for_general_tabular_MDP}.
\begin{restatable}{theorem}{ThmLBNormalMDP}[Lower Bound for $\algE$ without Gap Assumption]\label{thm:Regret_Lower_Bound_of_AlgP}
    There exist positive constants $c,\epsilon_0, \delta_0$, such that, for arbitrary $S \geq 4, A \geq 2, H \geq 2, K \geq \frac{c}{\epsilon_0^2}H^3SA$, and arbitrary algorithm pair $(\algO, \algE)$, there must exist a hard tabular MDP $M_{hard}$,
        $
        \EE_{(\algO, \algE), M_{hard}}\left[\sum_{k=1}^K V^* - V^{\piE_k}\right] \geq \delta_0\sqrt{cH^3SAK},
        $
    where the expectation is taken over the randomness of algorithms and MDP.
\end{restatable}
The theorem above is stating that, comparing with the regret lower bound for online algorithms $\tO(\sqrt{H^3SAK})$ in Theorem 9 of \citet{domingues2021episodic}, the exploitation algorithm cannot reduce the dependence on any of parameters $H,S,A,K$ in hard MDPs, even if we allow $\algO$ to sacrifice its performance to gather the best possible data for $\algE$. 
Also, the lower bound would still hold even if we allow both $\algO$ and $\algE$ to additionally use the data $\tauE$ generated by $\piE$.
This negative result implies that without any further assumptions, the separation is not beneficial from a minimax optimality perspective, and we can simply choose both $\algE$ and $\algO$ to be the same near-optimal online algorithm as  without worrying about separating them. 

However, in the next section,
we will show that, in tabular MDPs with strictly positive gaps, in contrast with the $\Omega(\log K)$ lower bound for online algorithms, we can have $\algE$ such that its regret is constant and independent on the number of time horizon $K$, which reveals the fundamental differences between the pure online setting and the Tiered RL setting considered in this paper.




\section{Pessimism in the Face of Uncertainty and Constant Regret}\label{sec:PVI_and_ConstantRegret}
In this section, we consider the gap-dependent setting and contribute to identifying the possibility to achieve constant regret by using pessimistic algorithms for $\algE$.
Intuitively, the main reason why PVI can lead to a constant regret is that the quality of the policy returned by PVI is positively correlated to the accumulation of optimal trajectories in the dataset $D$, which is directly connected with $\Regret(\algO)$.
As a result, on the one hand, the regret minimization objective of $\algE$ coincidentally aligns with the optimality constraint of $\algO$.
On the other hand, thanks to the positive gap assumption, $\piE$ will gradually converge to the optimal policy with high probability when $\algE$ is PVI, so there will be no regret after that.
In Sec. \ref{sec:analysis_for_MAB}, we start with stochastic MAB as a warm-up, and in Sec. \ref{sec:Tabular_RL} we extend our success to tabular RL setting. We defer the proofs in this section to Appendix~\ref{appx:analysis_for_bandit_setting}.
\subsection{Warm-Up: Gap-Dependent Regret Bound for Stochastic Multi-Armed Bandits}\label{sec:analysis_for_MAB}

\begin{algorithm}
    \textbf{Initilize}: $\alpha > 1;\quad N_i(1) \gets 0,~ \hat\mu_i(1)\gets 0,~\forall i \in \cA;\quad f(k):=1+16A^2(k+1)^2$\\
    \For{$k=1,2,...,K$}{
        $\piO_k \gets \arg\max_i \hat\mu_i(k) + \sqrt{\frac{2\alpha\log f(k)}{N_i(k)}},\quad\quad\piE_k \gets \arg\max_i \hat\mu_i(k) - \sqrt{\frac{2\alpha\log f(k)}{N_i(k)}}.$\\
        Interact with $\groupE$ and $\groupO$ by $\piE_k$ and $\piO_k$, and observe reward $r(\piE_k)$ and $r(\piO_k)$, respectively. \\
        \For{$i=1,2,...,A$}{
            $N_{i}(k+1) \gets N_i(k) + \mathbb{I}[\piO_k=i];\qquad \hat{\mu}_i(k+1) \gets \hat{\mu}_i(k)\frac{N_i(k)}{N_i(k+1)}+r(\piO_k)\frac{\mathbb{I}[\piO_k=i]}{{N_i(k+1)}}.$ 
        }
    }
    \caption{UCB-Exploration-LCB-Exploitation}\label{alg:UCB_Explore_LCB_Exploit}
\end{algorithm}
Our main algorithm for bandit setting is shown in Alg \ref{alg:general_learning_framework}, where we consider the UCB algorithm \citep{lattimore2020bandit} as $\algO$ and choose the LCB as $\algE$, which flips the sign of the bonus term in UCB.
We use $N_i(k)$ to denote the number of times that arm $i$ was pulled previous to step $k$, and use $\hat\mu_i$ to record the empirical average of arm $i$. 
Besides, we assume $1/N_i(\cdot)=+\infty$ if $N_i(\cdot)=0$, which implies that at the first $|\cA|$ steps the algorithm will pull each arm one by one.
Moreover, as we will show later, the choice of $\alpha > 1$ is crucial to avoiding dependence on $K$ in $\Regret(\algE)$ with our techniques. 
For Alg. \ref{alg:UCB_Explore_LCB_Exploit}, we have the following guarantee:
\begin{restatable}{theorem}{ThmLCBRegret}[Exploitation Regret]\label{thm:total_regret_UCB_LCB}
    In Algorithm \ref{alg:UCB_Explore_LCB_Exploit}, by choosing arbitrary $\alpha > 1$, there exists an absolute constant $c$, such that, for arbitrary $K \geq 1$, the pseduo-regret of $\algE$ is upper bounded by:
    $
        \Regret_K(\algE) \leq \tilde{O}\left(\frac{A}{\alpha - 1} +\alpha \sum_{\Delta_i > 0}  (A-i)\Big(\frac{1}{\Delta_i}- \frac{\Delta_i}{\Delta_{i-1}^2}\Big)\right)
    $ 
    where $\Delta_0 := \infty$ so $\frac{\Delta_1}{\Delta_{0}^2} = 0$.
\end{restatable}
Our result implies that by choosing PVI as $\algE$, we can achieve constant regret while keeping $\algO$ near-optimal.
Besides the advantages discussed in the related work paragraph in Sec.~\ref{sec:introduction}, our guarantee is also more favorable in certain cases compared to the $O(\sqrt{\frac{A}{\Delta_{\min}}}\sqrt{\sum_{\Delta_i > 0}\frac{1}{\Delta_i}})$ result in \citet{rouyer2020tsallis}: while 
it is not easy to verify whether our guarantee dominates theirs, in many cases ours can be strictly better (or at least no worse) than theirs. For example, consider the following two representative cases: $\Delta_1=\Delta_2=...\Delta_{A-1}=\Delta_{\min}$ (uniform gap) and $\Delta_1=\Delta_2=...=\Delta_{A-2}\gg\Delta_{A-1}=\Delta_{\min}$ (small last gap); our result achieves $\tO(\frac{A}{\Delta_{\min}})$ and $\tO(\frac{1}{\Delta_{\min}})$, respectively, in contrast to their $\tO(\frac{A}{\Delta_{\min}})$ and $\tO(\frac{\sqrt{A}}{\Delta_{\min}})$. 

\textbf{Proof Sketch}: The proof consists of two novel technique lemmas with a carefully chosen failure rate $\delta_k \sim O(1/k^{\Theta(\alpha)})$ so that the accumulative failure probability $\sum_{k=1}^\infty\delta_k<+\infty$. 
The first one is Lem. \ref{lem:blessing_of_pessimism}, where we show that w.p.~$1-\delta_k$, LCB will not prefer $i$ with $\Delta_i > 0$ as long as another better arm has been visited enough times in the dataset.
The second step is to identify a key property of UCB algorithm as stated in Lem. \ref{lem:upper_bound_of_Nk_geq_k_div_scalar}, where we provide a high probability upper bound that $N_i(k) \leq k / \lambda$ if $k\geq \tTheta({\lambda}/{\Delta_i^2})$ for arbitrary $\lambda \in [1, 4A]$, and it serves to indicate that the condition required by the success of LCB is achievable as long as $k$ is large enough 
\footnote{Comparing with results in Thm. 8.1 of \citep{lattimore2020bandit}, although our upper bounds of $N_i(k)$ is linear w.r.t. $k$ rather than $\log$ scale, we want to highlight that ours hold with high probability $O(1-k^{-\Theta(\alpha)})$ while \citep{lattimore2020bandit} only upper bounded the expectation.}.
\begin{restatable}{lemma}{LemBlessingLCB}[Blessing of Pessimism]\label{lem:blessing_of_pessimism}
    With the choice that $f(k)=1+16A^2(k+1)^2$, for arbitrary $i$ with $\Delta_i > 0$, for the LCB algorithm in Alg \ref{alg:UCB_Explore_LCB_Exploit}, and arbitrary $j$ satisfying $\Delta_j < \Delta_i$, we have:
        $
        \Pr\left(\{i = \piE_k\}\cap \{\Delta_j < \Delta_i\} \cap \left\{N_j(k) \geq \frac{8\alpha\log f(k)}{(\Delta_j - \Delta_i)^2}\right\}\right) \leq \frac{2}{k^{2\alpha}}.
        $
\end{restatable} 
\begin{restatable}{lemma}{LemLBofNiNew}[Property of UCB]\label{lem:upper_bound_of_Nk_geq_k_div_scalar}
    With the choice that $f(k)=1+16A^2(k+1)^2$, there exists a constant $c$, for arbitrary $i$ with $\Delta_i > 0$ and arbitrary $\scalar\in[1, 4A]$, in UCB algorithm, we have:
        $
        \Pr(N_i(k) \geq \frac{k}{\scalar}) \leq \frac{2}{k^{2\alpha-1}},\quad\forall k \geq \scalar + c\cdot\frac{\alpha \scalar}{\Delta_i^2} \log(1+\frac{\alpha A}{\Delta_{\min}}).
        $
\end{restatable}

Directly combining the above two results, we can obtain an upper bound for $\Regret(\algE)$ of order $\tO(A/\Delta_{\min}^{-2})$, which is already independent of $K$.
To achieve better dependence on $\Delta_{\min}$ in the regret, we conduct a finer analysis.
For each arm $i$ with $\Delta_i > 0$, we separate all the arms including $i$ into two groups based on whether its gap exceeds $\Delta_i / 2$: $\Gl_i=\{j:\Delta_j > \Delta_i / 2\}$ and $\Gu_i=\{j:\Delta_j \leq \Delta_i / 2\}$. 
As a result of Lem. \ref{lem:blessing_of_pessimism}, we know that $\piE_k$ will not prefer arm $i$ as long as there exists $j \in \Gu_i$ such that $N_j(k)=\tOmega(4\Delta^{-2}_i)=\tOmega(\Delta_i^{-2})$.
Based on Lem. \ref{lem:upper_bound_of_Nk_geq_k_div_scalar}, we know it is true with high probability, as long as $k \geq \tTheta(A\cdot\Delta_i^{-2})$, since at that time $N_l(k) \leq k / A$ holds for arbitrary $l\in \Gl_i$, which directly implies that $\max_{j:j\in \Gu_i} N_j(k)\geq \tOmega(\Delta_i^{-2})$. Then, combining Lem. \ref{lem:blessing_of_pessimism}, with high probability, the regret resulting from taken arm $i$ cannot be higher than $ \tTheta(A\cdot\Delta_i^{-2}) \cdot \Delta_i = \tTheta(A\cdot\Delta_i^{-1})$, which results in a $\tO(\sum_{\Delta_i > 0} A/\Delta_i)$ regret bound. 
As for the techniques leading to the further improvement in our final result, please refer to Lem. \ref{lem:combining_UCB_with_LCB_new} and the proof of Thm. \ref{thm:total_regret_UCB_LCB} in Appx. \ref{appx:analysis_for_bandit_setting}.

\subsection{Constant Regret of $\algE$ in Tabular MDPs}\label{sec:Tabular_RL}

In this section, we establish constant regret of $\algE$ based on realistic conditions for $\algO$ and $\algE$. 
We highlight the key steps of our analysis and our technical contributions here.

First of all, in Sec.\ref{sec:PVI_and_Property}, we propose the concrete PVI algorithm, and inspired by the clipping trick used for optimistic online algorithms \citep{simchowitz2019non}, we develop a high-probability gap-dependent upper bound for the sub-optimality of $\piE$, which is related to the accumulation of the optimal trajectories in dataset $D_k$.
Secondly, in Sec. \ref{sec:choice_algO}, we first introduce a general condition (Cond. \ref{cond:requirement_on_algO}) for the chocie of $\algO$,
based on which we quantify the accumulation of optimal trajectories in $D_k$ with the regret of $\algO$, and connect the exploration by $\algO$ and the optimality of $\algE$. We also supplement some details about how to relax such a condition and inherit the guarantees by the doubling-trick in Appx. \ref{appx:Doubling_Trick}, which may be of independent interest.
In Sec. \ref{sec:tabular_RL_main_analysis}, i.e. the last part of analysis, we bring the above two steps together and complete the proof. 
However, there is an additional challenge when the tabular MDP has multiple deterministic optimal policies, which is possible when there are non-unique optimal actions at some states. 
We overcome this difficulty by Thm.~\ref{thm:existence_of_well_covered_optpi} about policy coverage. 
To our knowledge, the only paper that runs into a similar challenge is \citep{papini2021reinforcement}, and they bypass the difficulty by assuming the uniqueness of optimal policy. Finally, Section \ref{sec:tabular_RL_Discussion} provide  some interpretation and implications of our results.

\subsubsection{Pessimistic Value Iteration as $\algE$ and its Property}\label{sec:PVI_and_Property}


\begin{algorithm}
    \textbf{Input}: Episode number $K$; Confidence level $\{\delta_k\}_{k=1}^K$; Bonus function $\textbf{Bonus}(\cdot,\cdot)$\\
    \For{$k=1,2,...,K$}{
        $\{b_{k,1}(\cdot,\cdot), b_{k,2}(\cdot,\cdot), ...,b_{k,H}(\cdot,\cdot)\}\gets \textbf{Bonus}(D_k, \delta_k).$ //\blue{Compute bonus function for PVI.}\\
        \For{$h = H,H-1,...,1$}{ 
            \For{$s_h \in \cS_h ,a_h \in \cA_h$}{
                $N_{k,h}(s_h,a_h)\gets$ the number of times $s_h,a_h$ occurs in the dataset $D_k$.\\
                $N_{k,h}(s_h,a_h,s_{h+1})\gets$ the number of times $(s_h,a_h,s_{h+1})$ occurs in the dataset $D_k$.\\
                $
                \hmP_{k,h}(\cdot|s_h,a_h)\gets 
                \begin{cases}
                    0,\quad & \text{if } N_{k,h}(s_h,a_h)=0;\\
                    \frac{N_{k,h}(s_h,a_h,\cdot)}{N_{k,h}(s_h,a_h)},\quad & \text{otherwise}.
                \end{cases}
                $\\
                }
            $\hQ_{k,h}(\cdot,\cdot) \gets \max\{R(\cdot,\cdot)+\hmP_{k,h}\hV_{k,h+1}(\cdot,\cdot) - b_{k,h}(\cdot,\cdot), 0\}.$ \\
            $\hV_{k,h}(\cdot) = \max_{a_h\in\cA}\hQ_{k,h}(\cdot, a_h),\quad \piPVI_{k,h}(\cdot)\gets \argmax_{a} \hQ_{k,h}(\cdot,a).$
        }
        $\piE_k \gets \{\piPVI_{k,1}, \piPVI_{k,2},...\piPVI_{k,H}\}$ \\
        // \blue{\textbf{Step 2}: Use $\algO$ satisfying Cond. \ref{cond:requirement_on_algO} to compute $\piO_k$ for $\groupO$} \\
        $\piO_k \gets \algO(D_k).$\\
        // \blue{\textbf{Step 3}: Sample trajectories and collect new data} \\
        Interact with $\groupE$ and $\groupO$ by $\piE_k$ and $\piO_k$, and observe $\tauE_k$ and $\tauO_k$, respectively. \\
        $D_{k+1} \gets D_k \cup \{\tauO_k\}$. 
    }
    \caption{Tiered-RL Algorithm with Pessimistic Value Iteration as $\algE$}\label{alg:PVI}
\end{algorithm}
The full details of our algorithm for tiered RL setting is provided in Alg. \ref{alg:PVI}, where we use PVI as $\algE$. 
Here we do not specify a concrete \textbf{Bonus} function, but provide general results for a range of qualified bonus functions satisfying Cond. \ref{cond:bonus_term} below. Cond. \ref{cond:bonus_term} can be satisfied by many bonus term considered in online literatures, and we briefly dicuss some examples in Appx.~\ref{appx:choice_of_bonus_term}.


\begin{condition}[Condition on Bonus Term for $\algE$]\label{cond:bonus_term}
    We define the following event at iteration $k\in[K]$ during the running of Alg. \ref{alg:PVI}:
        $
        \cEBk := \bigcap_{\substack{h\in[H],s_h\in\cS_h,a_h\in\cA_h}}\Big\{\{|\hat \mP_{k,h} \hV_{k,h+1}(s_h,a_h) - \mP_h\hV_{k,h+1}(s_h,a_h)| < b_{k,h}(s_h,a_h)\} \cap \{b_{k,h}(s_h,a_h) \leq B_1\sqrt{\frac{\log (B_2/\delta_k)}{N_{k,h}(s_h,a_h)}}\}\Big\}
        $
    where $B_1$ and $B_2$ are parameters depending on $S,A,H$ and $\Delta$ but independent of $\delta_k$, $k$.\footnote{Note that 
    we do not require the knowledge of $\Delta_i$'s to compute $b_{k,h}$.} We assume that, \textbf{Bonus} function satisfies that, in Alg. \ref{alg:PVI}, given arbitrary sequence $\{\delta_k\}_{k=1}^K$ with $\delta_1,\delta_2,...,\delta_K\in (0,1/2)$, at arbitrary iteration $k\in[K]$, we have $\Pr(\cEBk)\geq 1-\delta_k$.
\end{condition}
Next, we provide an upper bound for the sub-optimality gap of $\piPVI_k$ with the clipping operator $\Clip[x|\epsilon]:=x\cdot \mathbb{I}[x\geq \epsilon]$.
Previous upper bounds of PVI \citep[e.g.,~Theorem 4.4 of][]{jin2021pessimism} do not leverage the strictly positive gap and can be much looser when $N_{k,h}$ is large, and directly applying those results to our analysis would result in a regret scaling with $\sqrt{K}$.
\begin{restatable}{theorem}{ThmClipTrick}\label{thm:clipping_trick}
    By running Algorithm \ref{alg:PVI} with confidence level $\delta_k$, a function \textbf{Bonus} satisfying Condition \ref{cond:bonus_term}, and a dataset $D=\{\tau_1,...\tau_k\}$ consisting of $k$ complete trajectories generated by executing a sequence of policies $\pi_1,...,\pi_k$, on the event $\cEB$ defined in Condition \ref{cond:bonus_term}:
    \begin{align}
        V^{*}_1(s_1)-V_1^{\piPVI_k}(s_1) \leq 
        2\EE_{\pi^*}\left[\sum_{h=1}^H \Clip\left[\left. \min \left\{ H, 2B_1\sqrt{\frac{\log (B_2/\delta_k)}{N_{k,h}(s_h,a_h)}} \right\}\right|\epsClip \right] \right].
        \label{eq:sub_opt_clip_case1}
    \end{align}
    where $\pi^*$ 
    can be an arbitrary optimal policy, $\epsClip:= \frac{\Delta_{\min}}{2H+2}$ if $|\Pi^*|=1$ and $\epsClip:=\frac{d_{\min}\Delta_{\min}}{2SAH}$ if $|\Pi^*|>1$, where $d_{\min}:= \min_{\pi\in\Pi^*,h\in[H],s_h\in\cS_h,a_h\in\cA_h} d^{\pi}(s_h,a_h)$ subject to $d^{\pi}(s_h,a_h) > 0$.
\end{restatable}
\subsubsection{Choice and Analysis of $\algO$}\label{sec:choice_algO}
Next, we introduce our general condition for $\algO$ that the $\algO$ can achieve $O(\log k)$-regret with high probability. 
It is worth noting that many existing near-optimal online RL algorithms for gap-dependent settings may not directly satisfy the condition \citep{simchowitz2019non,xu2021fine,dann2021beyond} since they use a fixed confidence interval $\delta$.
In Appx.~\ref{appx:Doubling_Trick}, we will introduce a more realistic abstraction of those algorithms in Cond.~\ref{cond:realistic_requirement_on_algO}, and discuss in more details about how to close this gap with an algorithm framework inspired by the doubling trick. 
\begin{condition}[Condition on $\algO$]\label{cond:requirement_on_algO}
    $\algO$ is an algorithm which
    returns deterministic policies at each iteration, and for arbitrary $\nk\geq 2$, we have:
        $
        \Pr\big(\sum_{\tk=1}^\nk V^*_1(s_1)-V^{\piO_\tk}_1(s_1) > C_1 + \alpha C_2 \log \nk \big) \leq \frac{1}{\nk^\alpha},
        $
    where $C_1,C_2$ are parameters only depending on $S,A,H$ and gap $\Delta$ and independent of $\nk$.
\end{condition}
\textbf{Implication of Condition~\ref{cond:requirement_on_algO} for $\algO$}
Intuitively, low regret implies high accumulation of optimal trajectories in the dataset collected by $\algO$. We formalize this intuition in Thm.~\ref{thm:algO_regret_vs_algP_density} by establishing the relationship between the regret of $\algO$, $d^{\pi^*}$ and $\sum_{\tk=1}^k d^{\piO_\tk}(s_h,a_h)$ (the expectation of $N_{k,h}$).
\begin{restatable}{theorem}{ThmORegretPDensity}\label{thm:algO_regret_vs_algP_density}
    For an arbitrary sequence of deterministic policies $\pi_1,\pi_2,...,\pi_k$, there must exist a sequence of deterministic optimal policies $\pi_1^*, \pi_2^*,...,\pi_k^*$, such that $\forall h\in[H], s_h \in\cS_h, a_h\in\cA_h$:
    \begin{align*}
        \sum_{\tk=1}^k d^{\pi_\tk}(s_h,a_h) \geq \sum_{\tk=1}^k d^{\pi^*_\tk}(s_h,a_h) - \frac{1}{\Delta_{\min}} \Big(\sum_{\tk=1}^k V^*_1(s_1)-V^{\pi_\tk}_1(s_1)\Big).
    \end{align*}
\end{restatable}

\subsubsection{Main Results and Analysis}\label{sec:tabular_RL_main_analysis}
The main analysis is based on our discussion about the properties of $\algE$ and $\algO$ in previous sub-sections. 
In the following, we first discuss the proof sketch for the case when $|\Pi^*|=1$.
The main idea is to show that the unique $\pi^*$ will be ``well-covered'' by dataset, where we say a policy $\pi^*$ is ``well-covered'' if for each $(s_h,a_h)\in\cS_h\times\cA_h$ with $d^{\pi^*}(s_h,a_h) > 0$, 
$N_{k,h}(s_h,a_h)$ can strictly increase so that the RHS of Eq.\eqref{eq:sub_opt_clip_case1} in Thm.~\ref{thm:clipping_trick} will gradually decay to zero (e.g. $N_{k,h}(s_h,a_h)\geq \tilde{O}(k)$).
To show this, the key observation is that, 
with high probability, $N_{k,h}(s_h,a_h)$ will not deviate too much from its expectation $\sum_\tk d^{\pi_\tk}(s_h,a_h)$ (Lem. \ref{lem:concentration}), and can be lower bounded by $\sum_{\tk=1}^k d^{\pi_\tk^*}(s_h,a_h)-O(\log k)=kd^{\pi^*}(s_h,a_h)-O(\log k)$ as a result of Thm.~\ref{thm:algO_regret_vs_algP_density}.
As a result, the clipping operator in Eq.\eqref{eq:sub_opt_clip_case1} will take effects as long as $k$ is large enough, and $\piPVI_k$ will converge to the optimal policy with no regret. 
All that remains is to show the regret under failure events is also at the constant level because we choose a gradually decreasing failure rate $O(\frac{1}{k^\alpha})$, and $\lim_{K\rightarrow \infty}\sum_{k=1}^K O(\frac{1}{k^\alpha})<\infty$ as long as $\alpha > 1$. 

However, when $|\Pi^*| > 1$, the analysis becomes more challenging.
The main difficulty is that, when the optimal policy is not unique,
it is not obvious about the existence of ``well-covered'' $\pi^*$, since it is not guarantee that how much similarity is shared by the sequence of policies $\pi_1^*,...\pi^*_k$, especially when $|\Pi^*|$ is exponentially large (e.g. $|\Pi^*|=\Omega((SA)^H)$).
We overcome this difficulty by proving the existence of ``well-covered'' policy in the theorem stated below:
\begin{restatable}{theorem}{ThmExistWellCovered}[The existance of well-covered optimal policy]\label{thm:existence_of_well_covered_optpi}
    Given an arbitrary tabular MDP, and an arbitrary sequence of deterministic optimal policies $\pi^*_1,\pi^*_2,...\pi^*_\nk$ ($\pi^*_i$ may not equal to $\pi^*_j$ for arbitrary $1\leq i < j\leq \nk$ when there are multiple deterministic optimal policies), 
    there exists a (possibly stochastic) policy $\pi^*_{\text{cover}}$ 
    such that $\forall h\in[H],\forall (s_h,a_h)\in\cS_h\times\cA_h$ with $d^{\pi^*_{\text{cover}}}(s_h,a_h) > 0$:
    \begin{align*} 
        \sum_{\tk=1}^\nk d^{\pi^*_\tk}(s_h,a_h) \geq \frac{\nk}{2}\cdot \tilde{d}^{\pi^*_{\text{cover}}}(s_h,a_h),~ \text{with}~\tilde{d}^{\pi^*_{\text{cover}}}(\cdot,\cdot):= \max \left\{\frac{d^*_{h,\min}(\cdot,\cdot)}{(|\cZ_{h,\text{div}}|+1)H}, d^{\pi^*_{\text{cover}}}(\cdot,\cdot) \right\}.
    \end{align*}
    where $\cZ_{h,\text{div}}^* := \{(s_h,a_h)\in\cS_h\times\cA_h| \exists \pi^*, \tilde{\pi}^* \in \Pi^*,~ s.t.~ d^{\pi^*}(s_h) > 0,~ d^{\tilde{\pi}^*}(s_h) = 0\}$, and $d^*_{h,\min}(s_h,a_h):=\min_{\pi^*\in\Pi^*}d^{\pi^*}(s_h,a_h)$ subject to $d^{\pi^*}(s_h,a_h) > 0$.
\end{restatable}
Here we provide some explanation to the above result. According to the definition, $\cZ^*_{h,div}$ is the set including the state-action pairs which can be covered by some deterministic policies but is not reachable by some other deterministic policies, and therefore $|\cZ^*_{h,div}|\leq SA$ (or even $|\cZ^*_{h,div}| \ll SA$). Besides, $d^*_{h,\min}(s_h,a_h)$ denotes the minimal occupancy over all possible deterministic optimal policies which can hit $s_h,a_h$, and therefore, is no less than $d_{\min}$ defined in Thm.~\ref{thm:clipping_trick}. As a result, we know there exists a ``well-covered'' $\pi^*_{\text{cover}}$, since the accumulative density of its arbitrary reachable states can be lower bounded by $O(k)$. Then, following a similar discussion as the case $|\Pi^*|=1$,
we can finish the proof. We summarize our main result below.
\begin{restatable}{theorem}{ThmUBPVI}\label{thm:UB_Regret_PVI}
    By running an Algorithm satisfying Condition~\ref{cond:requirement_on_algO} as $\algO$, running Alg~\ref{alg:PVI} as $\algE$ with a bonus term function $\textbf{Bonus}$ satisfying Condition~\ref{cond:bonus_term} and $\delta_k = 1/k^\alpha$, for some constant $\alpha > 1$, for arbitrary $K \geq 1$, the exploitation regret of $\algE$ can be upper bounded by:

    (i) When $|\Pi^*| = 1$ (unique optimal deterministic policy):
    \begin{align*}
        \Regret_K(\algE) \leq &O\Big(\sum_{h=1}^H\sum_{\substack{s_h,a_h:\\d^{\pi^*}(s_h,a_h) > 0}} 
        \Big(\frac{C_1+C_2}{\Delta_{\min}}\log \frac{ SAH(C_1+C_2)}{d^{\pi^*}(s_h,a_h) \Delta_{\min}} + \frac{ B_1 H}{\Delta_{\min}}\log \frac{ B_2 H}{d^{\pi^*}(s_h,a_h)\Delta_{\min}}\Big)\Big).
    \end{align*}
    \indent (ii) When $|\Pi^*| > 1$ (non-unique optimal deterministic policies):
    \begin{align*}
        \Regret_K(\algE) \leq O\Big(\sum_{h=1}^H\sum_{\substack{s_h,a_h:\\d^{\pi^*_{\text{cover}}}(s_h,a_h) > 0}} \Big(\frac{C_1+C_2}{\Delta_{\min}}\log \frac{ SAH(C_1+C_2)}{\tilde{d}^{\pi^*_{\text{cover}}}(s_h,a_h)\Delta_{\min}} + \frac{ B_1 SAH}{d_{\min}\Delta_{\min}}\log \frac{ B_2 SAH}{d_{\min}\Delta_{\min}}\Big)\Big).
    \end{align*}
    where 
    $\pi^*_{\text{cover}}$ and $\tilde{d}^{\pi^*_{\text{cover}}}(s_h,a_h)$ are introduced in Theorem \ref{thm:existence_of_well_covered_optpi}.
\end{restatable}


\subsubsection{Interpretation of Results in Tabular RL}\label{sec:tabular_RL_Discussion}
Recall our objective in Sec.~\ref{sec:introduction} is to establish the benefits of leveraging the tiered structure by showing $\Regret_K(\algE)$ is constant. This contrasts the lower bound of online algorithms that continuously increases with the episode number $K$, which corresponds to the regret suffered by users in $\groupE$ without leveraging the tiered structure, while $\Regret_K(\algO)$ keeps (near-)optimal as before. In Appx.~\ref{appx:experiments}, we also provide some simulation results as a verification of our theoretical discovery.

One limitation of our results is that our bounds have additional dependence on $d^{\pi^*}$ (or even ${1}/{d^{\pi^*}}$) compared to most of the regret bounds in the online setting, although similar dependence on $\log d^{\pi^*}$ also appeared in a few recent works \citep[e.g., $\lambda_h^+$ in Thms. 8 and 9 of][]{papini2021reinforcement}. 
Besides, according to the lower and upper bound of online RL in gap-dependent settings \citep{simchowitz2019non},  $C_1+C_2$ in Cond.~\ref{cond:requirement_on_algO} have dependence on $O(\Delta_{\min}^{-1})$, which implies that in the regret bound in Thm.~\ref{thm:UB_Regret_PVI}, the dependence on $\Delta_{\min}$ would be $O(\Delta^{-2}_{\min})$.
For the former, in Appx.~\ref{appx:LB_Dep_Density}, we prove a lower bound, showing that $\log \frac{1}{d^{\pi^*}}$ is unavoidable when $\algO$ is allowed to behave adversarially without violating Cond.~\ref{cond:requirement_on_algO}; for the latter, we note that in the analysis of MAB setting (Sec.\ref{sec:analysis_for_MAB}), specifying the detailed behavior of $\algO$ can help tighten the bound.
Therefore, we conjecture that our results can be improved by putting more constraints on the behavior of $\algO$, which we leave to future work.

\section{Conclusion}
In this paper, we identify the tiered structure in many real-world applications and study the potential advantages of leveraging it by interacting with users from different groups with different strategies. Under the gap-dependent setting, we provide theoretical evidence of benefits by deriving constant regret for the exploitation policy while maintaining the optimality of the online learning policy.

As for the future work, we propose several potentially interesting directions.
\textbf{(i)} As we mentioned in Section~\ref{sec:tabular_RL_Discussion}, it is worth investigating the possibility of improving the regret bound of $\algE$ by considering a more concrete choice of $\algO$, or maybe other choices for $\algE$.
\textbf{(ii)} It would be interesting to relax our constraint on the optimality of $\algO$ by introducing the notion of budget $C$ as the tolerance on the sub-optimality of $\algO$. As a result, our setting and the decoupling exploration and exploitation setting can be regarded as special cases of a more general framework when $C=0$ and $C=\infty$. 
\textbf{(iii)} We assume that the users from different groups share the same transition and reward function, and it would also be interesting to extend our results to more general settings, where the group ID serves as context and will affect the dynamics \citep{abbasi2014online, modi2018markov}.
\textbf{(iv)} We only consider the setting with two tiers, and it may be worth studying the possibility and potential benefits under the setting with multiple tiers.

\section*{Acknowledgements}
JH's research activities on this work were conducted during his internship at MSRA.  
NJ's last involvement was in December 2021. NJ also acknowledges funding support from ARL Cooperative Agreement W911NF-17-2-0196, NSF IIS-2112471, NSF CAREER award, and Adobe Data Science Research Award. 
The authors thank Yuanying Cai for valuable discussion.

\bibliographystyle{plainnat}
\bibliography{references}

\newpage
\appendix


\newpage
\section{Detailed Related Work}\label{appx:detailed_related_work}
\paragraph{Online RL}
The online RL/MAB is the most basic framework studying the trade-off between exploration and exploitation \citep{auer2002finite, slivkins2019introduction, lattimore2020bandit}, where the agent targets at exploring the MDP to identify good actions as fast as possible to minimize the accumulative regrets.
In tabular MDPs \citep{jaksch2010near, dann2017unifying, jin2018q}, the 
regret lower bounds for non-stationary MDPs \citep{domingues2021episodic} have been achieved by \citet{azar2017minimax, zanette2019tighter}.
Recently, there has been interests in studying the gap-dependent regret  \citep{simchowitz2019non,xu2021fine,dann2021beyond,yang2021q, he2021logarithmic}, where the agent can achieve $\log$ dependence on the number of episodes $K$ under additional dependence on (the inverse of) the minimal gap $\frac{1}{\Delta_{\min}}$. 
\citet{simchowitz2019non} reports that, similar to stochastic MABs \citep{lattimore2020bandit}, the regret in the gap-dependent setting must scale as $\Omega(\log K)$, which implies that the $\log K$ upper bound is asymptotically tight. 
However, all of these works treat the customers equivalently and ignore the opportunities of leveraging the tiered structure. 
Recently, \citep{papini2021reinforcement} achieved similar constant regret in online setting with linear function approximation. Comparing with ours, they investigated the benefits of good features, while we focus on the benefits of considering a different learning protocol. Besides, although their results were established on a more general linear setting, their assumptions on the feature and the uniqueness of optimal policy are quite restrictive in tabular setting.

\paragraph{Offline RL}
Offline RL considers how to learn a good policy with a fixed dataset \citep{levine2020offline}.
Without the requirement of exploration, offline RL prefers algorithms with strong guarantees for exploitation and safety, and Pessimism in the Face of Uncertainty (PFU) becomes a major principle for achieving this both theoretically and empirically \citep{yin2021towards, uehara2022pessimistic, liu2020provably, xie2021bellman, buckman2020importance, kumar2020conservative, fujimoto2021minimalist, yu2020mopo}.
Similar to the offline setting, we choose $\algE$ to be a pessimistic algorithm. However, we still consider to interact the environment with $\algE$, although we ignore the data collected by $\algE$ for now and leave the investigation of its value to future work. As another difference, offline RL assumes the dataset is fixed and only the final performance matters, whereas we evaluate the accumulative regret of $\algE$. 

\section{More Discussion about Framework.~\ref{alg:general_learning_framework} and Motivating Examples}\label{appx:discussion_abs_framework_examples}
In this section, we try to justify that our Frw.~\ref{alg:general_learning_framework} is an appropriate abstraction for our motivating examples and user-interaction real-world applications. 

In the standard online learning protocol, at iteration $k$, the algorithm $\alg$ compute a policy $\pi_k$ based on previous exploration data, the environment samples a user $u_k$ from $\groupO$ and $\groupE$ according to the probability where $P(u_k \in \groupO)$ and $P(u_k \in \groupE)$, respectively (note that $P(u_k \in \groupO)+P(u_k \in \groupE)=1$).
After that, $\pi_k$ will interact with $u_k$ and obtain a new trajectory for the future learning, while $u_k$ suffers loss $V_1^*-V^{\pi_k}$, and the expected accumulative loss suffered by users from two groups till step $K$ is
\begin{align*}
    \Regret_K(\alg):=\EE[\sum_{k=1}^K V^*_1(s_1)-V^{\pi_k}(s_1)]
\end{align*}
Now, we consider a realistic assumption about the probability of $u_k$ from different groups:
\begin{assumption}[Assumption on the Ratio between Users from Different Groups]
We assume that:
\begin{align*}
    \frac{P(u_k \in \groupE)}{P(u_k \in \groupO)} = C,\quad \forall k\geq 1
\end{align*}
for some constant $C$.
\end{assumption}
Based on the assumption above, if we do not leverage the tiered structure, and treat the users from different groups equivalently, then the loss suffered by each group will be proportional to the size of that group. Therefore, even if we assume $\alg$ is near-optimal, the expected loss for each group will scale with $\log K$, i.e.:
\begin{align*}
    \text{Loss}_K(\groupO):=&\EE[\sum_{k=1}^K \mathbb{I}[u_k\in\groupO](V^*_1(s_1)-V^{\pi_k}(s_1))] \\
    =& \frac{1}{1+C}\EE[\sum_{k=1}^K V^*_1(s_1)-V^{\pi_k}(s_1)] = O(\frac{\log K}{1+C})\numberthis\label{eq:loss_K_gO}\\
    \text{Loss}_K(\groupE):=&\EE[\sum_{k=1}^K \mathbb{I}[u_k\in\groupE](V^*_1(s_1)-V^{\pi_k}(s_1))]\\
    =& \frac{C}{1+C}\EE[\sum_{k=1}^K V^*_1(s_1)-V^{\pi_k}(s_1)]=O(\frac{C\log K}{1+C})\numberthis\label{eq:loss_K_gP}
\end{align*}
Besides, we can also leverage the tiered structure, and consider an alternative protocol below:
\begin{algorithm}[H]
    \textbf{Initialize}: $D_1 \gets \{\},\quad k = 1,\quad\piO_1 \gets \algO(D_1),\quad \piE_1 \gets \algE(D_1)$ \\
    \For{$k=1,2,...$}{
        User $u_k$ comes.\\
        \If{$u_k \in \groupO$}{
            Use $\piO_k$ to interact with $u_k$, and collect data $\tauO_k$. \\
            $D_{k+1} = D_k \cup \{\tauO_k\} ,\quad \piO_{k+1} \gets \algO(D_k),\quad \piE_{k+1} \gets \algE(D_k)$ 
        }
        \Else{ 
            $\piE_k$ interacts with $u_k$, and collect data $\tauE_k$. // We do not use $\tauE_k$ for now.\\
            $D_{k+1} = D_k \cup \{\tauE_k\},\quad \piO_{k+1} \gets \piO_{k},\quad \piE_{k+1} \gets \piE_k$ 
        }
    }
    \caption{Online Interaction Protocol after Leveraging Tiered Structure}\label{alg:protocol_leverage_tiered_struc}
\end{algorithm}
In another word, in this new protocol, we use two policies at different exploitation level to interact with users from different groups, and only update policies if user comes from group $\groupO$. Note that in expectation, $\{u_k\in\groupO\}$ will happen for $\frac{K}{1+C}$ times, and therefore we have:
\begin{align*}
    \text{Loss}'_K(\groupO):=&\EE[\sum_{k=1}^K \mathbb{I}[u_k\in\groupO](V^*_1(s_1)-V^{\piO_k}(s_1))]\\
    \approx & \Regret_{K/(1+C)}(\algO)\\
    \text{Loss}'_K(\groupE):=&\EE[\sum_{k=1}^K \mathbb{I}[u_k\in\groupE](V^*_1(s_1)-V^{\piE_k}(s_1))]\\
    \approx & C\Regret_{K/(1+C)}(\algE)
\end{align*}
where $\Regret_{(\cdot)}(\algO)$ and $\Regret_{(\cdot)}(\algE)$ are originally defined in Def. \ref{def:Pseduo_regret}, and they are exactly the metric we used to measure the performance of $\algE$ and $\algO$ under our Frw.~\ref{alg:general_learning_framework}.

Based on our results in Sec. \ref{sec:analysis_for_MAB} and \ref{sec:Tabular_RL}, we know that under our framework, it is possible to achieve that:
\begin{align}
    \text{Loss}'_K(\groupO) =& \Regret_{K/(1+C)}(\algO) = O(\log \frac{K}{1+C})\label{eq:loss_prime_K_gO}\\
    \text{Loss}'_K(\groupE) =& C\Regret_{K/(1+C)}(\algE) = C\cdot \text{constant} \label{eq:loss_prime_K_gP}
\end{align}
where $constant$ means independence of $K$ but may include dependence on other parameters such as $S,A,H,\Delta$.
Comparing with Eq.\eqref{eq:loss_K_gO}, \eqref{eq:loss_K_gP}, \eqref{eq:loss_prime_K_gO}, and \eqref{eq:loss_prime_K_gP}, we can see that users from $\groupO$ will suffer less regret than before because we ``transfer'' some the regret from $\groupE$ to $\groupO$, and the additional regret suffered by $\groupO$ can be compensated in other forms as we discussed in Sec. \ref{sec:introduction}.

\paragraph{Remark}
Besides, our methods and results can be applied to those scenarios suggested by the decoupling setting \citep{avner2012decoupling,rouyer2020tsallis}, where $\algE$ does not necessarily interact with the environment, and we omit the discussion here.

\section{Lower Bounds}
\subsection{Regret Lower Bounds for Tabular MDP without Strictly Positive Gap Assumption}\label{appx:lower_bound_for_general_tabular_MDP}
We first recall a Theorem from \citep{dann2017unifying}:
\begin{theorem}[Theorem C.1 in \citep{dann2017unifying}]\label{thm:PAC_Lower_Bound}
    There exist positive constant $c$, $\delta_0 > 0$, $\epsilon_0 > 0$, such that for every $\epsilon \in (0, \epsilon_0)$, $S \geq 4$, $A\geq 2$ and for every algorithm $\text{Alg}$ and $n \leq \frac{cASH^3}{\epsilon^2}$  there is a fixed-horizon episodic MDP $M_{hard}$ with time-dependent transition probabilities and $S$ states and $A$ actions so that returning an $\epsilon$-optimal policy after n episodes is at most $1-\delta_0$.
\end{theorem}
\ThmLBNormalMDP*
\begin{proof}
    Suppose we have a pair algorithm $(\algO, \algE)$, we can construct a PAC algorithm with $\algO$ and $\algE$, in the following way:
    \begin{itemize}
        \item Input: $K$.
        \item For $k=1,2,...,K$, run $\algO$ to collect data and run $\algE$ to generate a sequence of policies $\piE_1,...,\piE_K$.
        \item Uniformly randomly select an index from $\{1,2,...,K\}$, and denote it as $K_{\text{PAC}}$
        \item Output $\piE_{K_{\text{PAC}}}$.
    \end{itemize}
    In the following, we denote such an algorithm as $\text{Alg}_{\text{PAC}}$. Then, for an arbitrary MDP $M$, we must have:
    \begin{align*}
        \EE_{\text{Alg}_{\text{PAC}}, M}[V^* - V^{\pi_{K_\text{PAC}}}] = \frac{1}{K} \EE_{(\algO, \algE), M}[\sum_{k=1}^K V^* - V^{\pi_{k}}]
    \end{align*}
    As a result of Markov inequality and that $V^*-V^\pi \geq 0$ for arbitrary $\pi$, for arbitrary $\epsilon > 0$, we have:
    \begin{align}
        \Pr(V^* - V^{\pi_{K_{\text{PAC}}}} \geq \epsilon) \leq \frac{\EE_{\text{Alg}_{\text{PAC}}, M}[V^* - V^{\pi_{K_{\text{PAC}}}}]}{\epsilon}\label{eq:markov_ineq}
    \end{align}
    Since the above holds for arbitrary $\epsilon \in (0,\epsilon_0)$, by choosing $\epsilon=\bar\epsilon := \sqrt{{cH^3SA}/{K}}$, since $K > \frac{c}{\epsilon_0^2}H^3SA$, we have $\bar\epsilon < \epsilon_0$ and
    \begin{align}
        \Pr(V^* - V^{\pi_{K_{\text{PAC}}}} \geq \bar\epsilon) \leq \frac{\EE_{\text{Alg}_{\text{PAC}}, M}[V^* - V^{\pi_{K_{\text{PAC}}}}]}{\sqrt{cH^3SA/K}}.\label{eq:high_prob_PAC}
    \end{align}
    Because $K \leq cH^3SA/{\bar\epsilon}^2$, Thm. \ref{thm:PAC_Lower_Bound} implies that, for $\text{Alg}_{\text{PAC}}$, for arbitrary $S \geq 4, A \geq 2, H \geq 1$, there must exists a hard MDP $M_{hard}$, such that:
    \begin{align*}
        \frac{\EE_{\text{Alg}_{\text{PAC}}, M}[V^* - V^{\pi_{K_{\text{PAC}}}}]}{\sqrt{cH^3SA/K}} \geq \delta_0
    \end{align*}
    and it is equivalent to:
    \begin{align}
        \EE_{(\algO, \algE), M}[\sum_{k=1}^K V^* - V^{\pi_{k}}]=K\cdot \EE_{\text{Alg}_{\text{PAC}}, M}[V^* - V^{\pi_{K_{\text{PAC}}}}] \geq \delta_0\sqrt{cH^3SAK}\label{eq:lower_bound_induced_by_PAC}
    \end{align}
    which finishes the proof.

\end{proof}
\begin{remark}[Regret lower bound when $\tauE$ is used in Framework \ref{alg:general_learning_framework}]
    Our techniques can be extended to the case when $\tauE$ are used by $\algO$ and $\algE$, and establish the same $O(\sqrt{H^3SAK})$ lower bound. Because the only difference would be in this new setting, at iteration $k$, $\algO$ and $\algE$ will use $2k$ trajectories to compute $\piO_k$ and $\piE_k$, which only double the sample size comparing with $k$ trajectories in Framework \ref{alg:general_learning_framework}. Therefore, with the same techniques (and choosing $\bar{\epsilon}=\sqrt{{cH^3SA}/{(2K)}}$), one can obtain a lower bound which differs from Eq.\eqref{eq:lower_bound_induced_by_PAC} by constant.
\end{remark}

\subsection{Lower Bound for the Dependence on $\log d_{\min}$ when $|\Pi^*|=1$}\label{appx:LB_Dep_Density}
In this section, because we will conduct discussion on multiple different MDPs, and in order to distinguish them, we will introduce a subscript of $M$ to highlight which MDP we are discussing. Therefore, we revise some key notations and re-introduce them here.
\paragraph{Notation}
Given arbitrary tabular MDP $M=\{\cS,\cA,\mP,r,H\}$. We use $\Pi^*_M$ to denote the set of deterministic optimal policies of $M$. For each deterministic optimal policy $\pi^*_M \in \Pi^*_M$, we define $d^{\pi^*_M}_{\min}$ to be the minimal non-zero occupancy of the reachable state action by $\pi^*_M$, i.e.
\begin{align}
    d^{\pi^*_M}_{\min}:= \min_{h,s_h,a_h} d^{\pi^*_M}(s_h,a_h),\quad s.t.\quad d^{\pi^*_M}(s_h,a_h) > 0 \label{eq:def_of_d_pistar_min}
\end{align}
Then, we define:
\begin{align}
    d_{M,\min}:= \min_{\pi^*_M\in \Pi^*_M} d^{\pi^*_M}_{\min},\quad \pi^*_{M,d_{\min}}:=\arg\min_{\pi^*_M\in \Pi^*_M} d^{\pi^*_M}_{\min},\label{eq:def_of_d_min}
\end{align}
Different from regret analysis in online setting \citep{simchowitz2019non,xu2021fine,dann2021beyond}, our regret bound in Thm. \ref{thm:UB_Regret_PVI} has additional dependence on $\log d^{\pi^*}$, even if the optimal policy is unique. 
In Thm. \ref{thm:LB_Dep_Density_Formal}, we show that the dependence of $\log d^{\pi^*}$ is unavoidable if we do not have other assumptions about the behavior of $\algO$ besides Cond. \ref{cond:requirement_on_algO}. 
In another word, even if constrained by satisfies Cond. \ref{cond:requirement_on_algO}, $\algO$ can be arbitrarily adversarial so that $\log d^{\pi^*}$ exists in the lower bound. We defer the proof to Appx. \ref{appx:LB_Dep_Density}

\begin{restatable}{theorem}{ThmLBDensityDep}\label{thm:LB_Dep_Density_Formal}
    For arbitrary $S,A,H \geq 3$, arbitrary $\Delta_{\min} > 0$ and $d_{\min} > 0$, if there exists an MDP $M=\{\cS,\cA,\mP,r,H\}$ such that $|\cS|=S$, $|\cA|=A$ $d_{M,\min}=d_{\min}$ and the minimal gap is lower bound by $\Delta_{\min}$, then there exists a hard MDP $M^+=\{\cS^+,\cA,\mP^+,r^+,H\}$ with $|\cS^+|=S+1$, minimal gap lower bounded by $3\Delta_{\min}/4$ and $d_{M^+,\min} = d_{M,\min}/4$, and an adversarial choice of $\algO$ satisfying Cond. \ref{cond:requirement_on_algO}, such that when $K$ is large enough, the expected Pseudo-Regret of $\algE$ is lower bounded by:
    \begin{align*}
        \EE_{\algO, M, \algE}[\sum_{k=1}^K V^* - V^{\piE_k}] \geq O\Big((C_1+C_2)\log\frac{C_1+C_2}{d_{M^+,\min}\Delta_{\min}}\Big)
    \end{align*}
\end{restatable}
\begin{proof}
The proof is divided into three steps.
\paragraph{Step 1: Construction of the Hard MDP Instance}

Now, we construct a hard MDP instance $M^+:=\{\cS^+,\cA^+,\mP^+,r^+, H\}$ based on $M$ by expanding the state space with an absorbing state $s_{absorb}$ for layer $h\geq 2$ (we use $h$ in $s_{h,absorb}$ to distinguish the absorbing state at different time step), and define the transition and reward function by:
\begin{align*}
    \forall a_1 \in \cA_1,s_2\in\cS_2,\quad & \mP^+(s_{2,absorb}|s_1,a_1)=\frac{d_{M,\min}}{4}, \\
    & \mP^+(s_2|s_1,a_1)=(1-\frac{d_{M,\min}}{4})\mP(s_2|s_1,a_1)\\
    & r^+(s_1,a_1)=(1-\frac{d_{M,\min}}{4}) r(s_1,a_1) \\
    \forall h\geq 2, s_h\in\cS_h,a_h\in\cA_h,\quad & \mP^+(\cdot|s_h,a_h)=\mP(\cdot|s_h,a_h),\quad r^+(s_h,a_h)=r(s_h,a_h) \\
    \forall h \geq 2, a_h\in\cA_h, \quad & \mP^+(s_{h+1,absorb}|s_{h,absorb},a_h)=1 \\
    \forall~H\geq h \geq 2,  a_h\in\cA_h, \quad & r(s_{h,absorb},a_h)=\Delta_{\min} \mathbb{I}[a_h=a_h^*]
\end{align*}
Briefly speaking, at the initial state, by taking arbitrary action, with probability $d_{M,\min}/2$, it will transit to absorbing state at layer 2, and the agent can not escape from the absorbing state till the end of the episodes. Besides, at the absorbing states, for each layer $2\leq h\leq H$, there always exists an optimal action $a_h^*$ with reward $\Delta_{\min}$ and taking any the other actions will lead to 0 reward. 
Moreover, $M^+$ agrees with $M$ for all the transition and rewards when $h\geq 2$.

Easy to see that:
\begin{align*}
    V^*_{M^+}(s_1) - Q^*_{M^+}(s_1,a_1) = (1-\frac{d_{M,\min}}{4}) (V^*_{M}(s_1) - Q^*_M(s_1,a_1))
\end{align*}
Therefore, if $V^*_{M}(s_1) - Q^*_M(s_1,a_1) > 0$, we still have:
\begin{align*}
    V^*_{M^+}(s_1) - Q^*_{M^+}(s_1,a_1) \geq \frac{3}{4}(V^*_{M}(s_1) - Q^*_M(s_1,a_1)) \geq \frac{3}{4}\Delta_{\min}
\end{align*}
Combining with the transition and reward functions in absorbing states, we can conclude that the gap of $M^+$ is still $O(\Delta_{\min})$.

\paragraph{Step 2: Construction of Adversarial $\algO$}
Let's use $\Pi^*_{M^+}$ to denote the set of deterministic optimal policies at MDP $M^+$. It's easy to see that, for arbitrary $\pi^*_{M^+} \in \Pi^*_{M^+}$, there must exists an optimal policy $\pi^*_M \in \Pi^*_M$ agrees with $\pi^*_{M^+}$ at all non-absorbing states (and vice versa), i.e.
\begin{align*}
    \pi^*_{M^+}(s_h) = \pi^*_M(s_h),\quad\forall h \in [H], s_h \in \cS_h
\end{align*}
Then, for arbitrary $\pi^*_{M^+}\in \Pi^*_{M^+}$, we have:
\begin{align*}
    d^{\pi^*_{M^+}}(s_{h,absorb}) = \frac{d_{M,\min}}{4} \leq (1-\frac{d_{M,\min}}{4})d_{M,\min} < d^{\pi^*_{M^+}}(s_{h'}),\quad \forall h'\in[H],s_{h'}\in\cS_{h'}
\end{align*}
which implies that 
\begin{align*}
    d_{M^+,\min} = \frac{d_{M,\min}}{4}
\end{align*}
and $s_{h,absorb}$ are the hardest state to reach for all deterministic optimal policies. In the following, we randomly choose an optimal deterministic policy $\pi^*_{M^+}$ from $\Pi^*_{M^+}$, and randomly select one action $\bar{a}_H$ from $\cA_H$ with $\bar{a}_H\neq a_H^*$ and fix them in the following discussion.

Based on the definition above, we define a deterministic policy in $M^+$, which agree with $\pi^*_{M^+}$ for all states except $s_{H}$:
\begin{align*}
    \forall h \in [H],\quad  \pi_{M^+}(s_h) = \begin{cases}
        \pi^*_{M^+}(s_h),\quad & \text{if } s_h\neq s_{H,absorb},\\
        \bar{a}_H,\quad & \text{if } s_h=s_{H,absorb},
    \end{cases}.
\end{align*}
Now, we are ready to design the adversarial choice of $\algO$ satisfying the condition \ref{cond:requirement_on_algO}. We consider the following algorithm:
\begin{align*}
    \algO(k) = \begin{cases}
        \pi_{M^+},\quad & \text{if } k \leq k_{\sup},\\
        \pi^*_{M^+},\quad & \text{if } k > k_{\sup},
    \end{cases};\quad 
\end{align*}
where $k_{\sup}$ is defined to be:
\begin{align*}
    k_{\sup}:=\sup_{k\in N^+}: \{k\leq \frac{1}{d_{M^+,\min}\Delta_{\min}}(C_1+C_2\log k)\} \approx O(\frac{C_1+C_2}{d_{M^+,\min}\Delta_{\min}}\log \frac{C_1+C_2}{d_{M^+,\min}\Delta_{\min}})
\end{align*}
We can easily verify that Cond. \ref{cond:requirement_on_algO} will not be violated, since
\begin{align*}
    \forall k \geq 1,\quad \quad&\sum_{k=1}^K V^*-V^{\piO_k} \\
    \leq& d_{M^+,\min}\cdot (V^*(s_{2,absorb})-V^{\piO_k}(s_{2,absorb}))\cdot \min\{k,k_{\sup}\} \\
    \leq& d_{M^+,\min}\cdot (V^*(s_{H,absorb})-V^{\piO_k}(s_{H,absorb}))\cdot \frac{1}{d_{M^+,\min}\Delta_{\min}}(C_1+C_2\log \min\{k,k_{\sup}\}) \\
    =&d_{M^+,\min}\cdot \Delta_{\min}\cdot \frac{1}{d_{M^+,\min}\Delta_{\min}}(C_1+C_2\log k) \\
    \leq& C_1 + C_2 \log k
\end{align*}

\paragraph{Step 3: Lower Bound of $\algE$ under the Choice of Adversarial $\algO$}
Now, we can derive an lower bound for $\algE$. Since in the first $k_{\sup}$ steps, $\algE$ can only observe what happens if action $\bar{a}_H$ is taken at $s_{H,absorb}$, and therefore, it has no idea about which action among $\cA_H \setminus \bar{a}_H$ is the optimal action $a_H^*$. We use $\mathcal{M}^+$ to denote a set of MDPs by permuting the position of $a_H^*$ in $M^+$. Since $|\cA_H|=A$, we have $|\mathcal{M}^+|=A-1$ and $M^+ \in \mathcal{M}^+$. 

Then, we uniformly sample an MDP from $\mathcal{M}^+$ and run the adversarial $\algO$ above to generate the data for $\algE$ to learn. We use $M^+_i$ with $i=1,2...,A-1$ to refer to the MDPs in $\mathcal{M}^+$ and use index $i$ to refer to the position of the optimal action at $s_{H,absorb}$ in each MDP. For the simplicity of the notation, we use $A$ as the index to refer to the position of $\bar{a}_H$.

Because $\algE$ do not have prior knowledge about which MDP in $\mathcal{M}^+$ is sampled, we have:
\begin{align*}
    &\EE_{\bar{M}^+, \algO, \algE}[\sum_{k=1}^K V^* - V^{\piE_k}] \\
    \geq &\EE_{\bar{M}^+, \algO, \algE}[\sum_{k=1}^{k_{\sup}} V^* - V^{\piE_k}] \\
    = & \frac{1}{A-1}\sum_{i\in[A-1]} \EE_{M_i^+,, \algO, \algE}[\sum_{k=1}^{k_{\sup}} V^* - V^{\piE_k}]\\
    \geq&\frac{d_{M,\min}\Delta_{\min}}{A-1}\sum_{k=1}^{k_{\sup}}\sum_{i\in[A-1]} \sum_{j \in [A], j\neq i} {\text{Pr}}_{\algO, M_i}(\piE_k(s_{H,absorb})=j) \tag{Drop the probability that $\piE_k$ is sub-optimal at non-absorbing states}\\
    =&\frac{d_{M,\min}\Delta_{\min}}{A-1}\sum_{k=1}^{k_{\sup}}\Big(\sum_{j \in [A], j\neq A-1} {\text{Pr}}_{\algO, M_i}(\piE_k(s_{H,absorb})=j)\\
    &\qquad\qquad\qquad+\sum_{i\in[A-2]} \sum_{j \in [A], j\neq i} {\text{Pr}}_{\algO, M_i}(\piE_k(s_{H,absorb})=j)\Big)\\
    =&\frac{d_{M,\min}\Delta_{\min}}{A-1}\sum_{k=1}^{k_{\sup}}\Big(\sum_{i\in[A-2]} {\text{Pr}}_{\algO, M_i}(\piE_k(s_{H,absorb})=i)\\
    &\qquad\qquad\qquad+\sum_{i\in[A-2]} \sum_{j \in [A], j\neq i} {\text{Pr}}_{\algO, M_i}(\piE_k(s_{H,absorb})=j)\Big) \tag{$\algE$ can not distinguish between $M^+_i$}\\
    =&\frac{d_{M,\min}\Delta_{\min}}{A-1}\sum_{k=1}^{k_{\sup}}\Big(\sum_{i\in[A-2]} \sum_{j \in [A]} {\text{Pr}}_{\algO, M_i}(\piE_k(s_{H,absorb})=j)\\
    =& \frac{A-2}{A-1} d_{M,\min}\Delta_{\min} k_{\sup}\\
    =& O\Big((C_1+C_2)\log\frac{C_1+C_2}{d_{M^+,\min}\Delta_{\min}}\Big) = O\Big((C_1+C_2)\log\frac{C_1+C_2}{d_{M,\min}\Delta_{\min}}\Big)
\end{align*}    
\end{proof}

\section{Analysis for Bandit Setting}\label{appx:analysis_for_bandit_setting}


\subsection{The Optimality of $\algO$ in Alg \ref{alg:UCB_Explore_LCB_Exploit}}\label{sec:optimality_of_algO}
From Theorem 8.1 of \citep{lattimore2020bandit}, 
we can show the following guarantee for the UCB algorithm in Alg \ref{alg:UCB_Explore_LCB_Exploit} with revised bonus function:
\begin{align*}
    \EE[\sum_{k=1}^K \mu_1- \mu_{\piO_k}] \leq& \sum_{i: \Delta_i > 0} \Delta_i + \frac{1}{\Delta_i}(8\alpha\log f(K) + 8\sqrt{\pi \alpha \log f(K)} + 28)\\
    =&O(\sum_{i:\Delta_i > 0}\frac{\alpha}{\Delta_i} \log AK)
\end{align*}
where we assume $K \geq A$. Since $\alpha$ is just at the constant level, the above regret still matches the lower bound.

\subsection{Analysis for LCB}
\paragraph{Outline} 
In this section, we establish regret bound for Alg. \ref{alg:UCB_Explore_LCB_Exploit}. We first provide the proof of two key Lemma: Lem. \ref{lem:blessing_of_pessimism} and Lem. \ref{lem:upper_bound_of_Nk_geq_k_div_scalar}. After that, in Lem. \ref{lem:upper_bound_of_Nk_geq_k_div_scalar}, we try to combine the above two results and prove that for those arm $i$ with $\Delta_i > 0$, when $k$ is large enough, we are almost sure (with high probability) that LCB will not take arm $i$.
Finally, we conclude this section with the proof of Thm. \ref{thm:total_regret_UCB_LCB}.

\paragraph{Definition of $c_f$}
Under the choice $f(k)=1+16A^2(k+1)^2$, we use $c_f$ to denote the minimal positive constant independent with $\alpha, A$ and arbitrary $\Delta_i$ with $\Delta_i > 0$, such that
\begin{align*}
    \forall~\Delta_i > 0, \forall \scalar \in [1, 4A],\quad \text{as long as}\quad  k \geq& c_f \frac{\alpha \scalar}{\Delta_i^2} \log (1+\frac{\alpha A}{\Delta_i}), \quad \\
    \text{we have}\quad k \geq& \frac{32\alpha \scalar}{\Delta_i^2} \log f(k) \numberthis\label{eq:def_c_f_new}
\end{align*}


\LemBlessingLCB*
\begin{proof}
    \begin{align*}
        & \Pr(\{i = \piE_k\}\cap \{\Delta_j < \Delta_i\} \cap \{N_j(k) \geq \frac{8\alpha\log f(k)}{(\Delta_j - \Delta_i)^2}\})\\
        \leq & \Pr(\{\hat\mu_i(k)-\sqrt{\frac{2\alpha \log f(k)}{N_i(k)}} \geq \hat\mu_{j}(k)-\sqrt{\frac{2\alpha \log f(k)}{N_{j}(k)}}\}\cap \{\Delta_j < \Delta_i\} \cap \{N_j(k) \geq \frac{8\alpha\log f(k)}{(\Delta_j - \Delta_i)^2}\}) \\
        = & \Pr(\{\hat\mu_i(k)-\mu_i-\sqrt{\frac{2\alpha \log f(k)}{N_i(k)}} \geq \hat\mu_{j}(k)-\mu_j + (\Delta_j - \Delta_i)-\sqrt{\frac{2\alpha \log f(k)}{N_{j}(k)}}\}\\
        &\quad\quad\cap \{\Delta_j < \Delta_i\} \cap \{N_j(k) \geq \frac{8\alpha\log f(k)}{(\Delta_j - \Delta_i)^2}\})\\
        \leq& \Pr(\{\hat\mu_i(k)-\mu_i-\sqrt{\frac{2\alpha \log f(k)}{N_i(k)}} \geq \hat\mu_{j}(k)-\mu_j + \sqrt{\frac{2\alpha \log f(k)}{N_{j}(k)}}\}\cap \{\Delta_j < \Delta_i\} \cap \{N_j(k) \geq \frac{8\alpha\log f(k)}{(\Delta_j - \Delta_i)^2}\})\\
        \leq&\Pr(\{\hat\mu_i(k)-\mu_i-\sqrt{\frac{2\alpha \log f(k)}{N_i(k)}} \geq \hat\mu_{j}(k)-\mu_j + \sqrt{\frac{2\alpha \log f(k)}{N_{j}(k)}}\})\\
        \leq& \Pr(\{\hat\mu_i(k)-\mu_i-\sqrt{\frac{2\alpha \log f(k)}{N_i(k)}} \geq 0\}) + \Pr(\{0 \geq \hat\mu_{j}(k)-\mu_j + \sqrt{\frac{2\alpha \log f(k)}{N_{j}(k)}}\}) \\
        \leq& 2/f(k)^\alpha \leq 2/k^{2\alpha}
    \end{align*}
    where the last but two step is because of the Azuma-Hoeffding's inequality.
\end{proof}

\LemLBofNiNew*
\begin{proof}
    We choose $c_f$ defined in Eq.\eqref{eq:def_c_f_new} to be the constant $c$ in this Lemma.

    The key idea of the proof is that, because $N_i(k) \leq k$ for all $k$, if $N_i(k) \geq \lceil k / \scalar \rceil $, there must exists an iteration $\tk$ between $\lceil k/\scalar \rceil - 1$ and $k$, such that $\{N_i(\tk)=\lceil k/\scalar \rceil-1\}\cap \{N_i(\tk)=\lceil k/\scalar \rceil\}$ (i.e. $\tk$ is the time step that UCB takes arm $i$ for the $\lceil k/\scalar \rceil$-th time). Therefore, for arbitrary fixed $\lambda \in [1, A^2]$, when $k \geq \scalar + c_f \cdot \frac{\alpha \scalar}{\Delta_i^2}\log(1+\frac{\alpha \scalar}{\Delta_{\min}})$, we have:

    \begin{align*}
        &\Pr(N_i(k) \geq k/\scalar)=\Pr(N_i(k) \geq \lceil k/\scalar \rceil) \\
        =&\sum_{\tilde{k}=\lc k/\scalar \rc-1}^{k-1} \Pr(\{N_i(\tilde{k})=\lc k/\scalar \rc-1,N_i(\tilde{k}+1)=\lc k/\scalar \rc\}\cap \{\hat{\mu}_{i^*}(\tilde{k})+\sqrt{\frac{2\alpha \log f(\tilde{k})}{N_{i^*}(\tilde{k})}} \leq \hat\mu_i(\tilde{k})+\sqrt{\frac{2\alpha \log f(\tilde{k})}{N_i(\tilde{k})}}\}) \tag{Union bound.}\\
        =&\sum_{\tilde{k}=\lc k/\scalar \rc-1}^{k-1} \Pr(\{N_i(\tilde{k})=\lc k/\scalar \rc-1,N_i(\tilde{k}+1)=\lc k/\scalar \rc\}\\
        &\quad\quad\cap \{\hat{\mu}_{i^*}(\tilde{k})-\mu_{i^*}+\sqrt{\frac{2\alpha \log f(\tilde{k})}{N_{i^*}(\tilde{k})}} \leq \hat\mu_i(\tilde{k})-\mu_i-\Delta_i+\sqrt{\frac{2\alpha \log f(\tilde{k})}{N_i(\tilde{k})}}\}) \tag{Subtract $\mu_{i^*}$ at both sides}\\
        \leq& \sum_{\tilde{k}=\lc k/\scalar \rc-1}^{k-1} \Pr(\{N_i(\tilde{k})=\lc k/\scalar-1 \rc,N_i(\tilde{k}+1)=\lc k/\scalar \rc\}\cap\{\hat{\mu}_{i^*}(\tilde{k})-\mu_{i^*}+\sqrt{\frac{2\alpha \log f(\tilde{k})}{N_{i^*}(\tilde{k})}}\leq 0\})\\
        &\quad + \Pr(\{N_i(\tilde{k})=\lc k/\scalar \rc-1),N_i(\tilde{k})=\lc k/\scalar \rc+1)\}\cap\{0 \leq \hat\mu_i(\tilde{k})-\mu_i-\Delta_i+\sqrt{\frac{2\alpha \log f(\tilde{k})}{N_i(\tilde{k})}}\})\numberthis\label{eq:med_step_1}\\
        \leq& \sum_{\tilde{k}=\lc k/\scalar \rc-1}^{k-1} \Pr(\{\hat{\mu}_{i^*}(\tilde{k})-\mu_{i^*}+\sqrt{\frac{2\alpha \log f(\tilde{k})}{N_{i^*}(\tilde{k})}}\leq 0\})\\
        &\quad + \Pr(\{N_i(\tilde{k})=\lc k/\scalar \rc-1),N_i(\tilde{k}+1)=\lc k/\scalar \rc)\}\cap\{0 \leq \hat\mu_i(\tilde{k})-\mu_i-\sqrt{\frac{2\alpha \log f(\tilde{k})}{N_i(\tilde{k})}}\}) \tag{Under our choice of $k$, and $N_i(\tilde{k})=\lc k/\scalar\rc-1$, $\frac{\log f(\tilde{k})}{N_i(\tilde{k})} \leq \frac{\log f(k)}{ k/\scalar - 1}\leq \frac{\Delta_i^2}{8\alpha}$}\\
        \leq&\sum_{\tilde{k}=\lc k/\scalar \rc-1}^{k-1} \Pr(\{\hat{\mu}_{i^*}(\tilde{k})-\mu_{i^*}+\sqrt{\frac{2\alpha \log f(\tilde{k})}{N_{i^*}(\tilde{k})}}\leq 0\}) + \Pr(\{0 \leq \hat\mu_i(\tilde{k})-\mu_i-\sqrt{\frac{2\alpha \log f(\tilde{k})}{N_i(\tilde{k})}}\})\\
        \leq& \sum_{\tilde{k}=\lc k/\scalar \rc-1}^{k-1} \frac{2}{f(\tilde{k})^{\alpha}} \tag{Azuma-Hoeffding Inequality}\\
        \leq& \frac{2k}{f(k/\scalar-1)^{\alpha}} = \frac{2k}{(16A^2k^2/\scalar^2+1)^{\alpha}}\leq \frac{2}{k^{2\alpha-1}} \tag{$\scalar\in[1, 4A]$}
    \end{align*}
    where the step \eqref{eq:med_step_1} is because:
    \begin{align*}
        &\{\hat{\mu}_{i^*}(\tilde{k})-\mu_{i^*}+\sqrt{\frac{2\alpha \log f(\tilde{k})}{N_{i^*}(\tilde{k})}} < \hat\mu_i(\tilde{k})-\mu_i-\Delta_i+\sqrt{\frac{2\alpha \log f(\tilde{k})}{N_i(\tilde{k})}}\}\\
        \in& \{0 < \hat\mu_i(\tilde{k})-\mu_i-\Delta_i+\sqrt{\frac{2\alpha \log f(\tilde{k})}{N_i(\tilde{k})}}\} \cup \{\hat{\mu}_{i^*}(\tilde{k})-\mu_{i^*}+\sqrt{\frac{2\alpha \log f(\tilde{k})}{N_{i^*}(\tilde{k})}}\leq 0\}
    \end{align*}
\end{proof}

\begin{lemma}\label{lem:combining_UCB_with_LCB_new}
    Given an arm $i$, we separate all the arms into two parts depending on whether its gap is larger than $\Delta_i$ and define $\Gl_i:=\{\iota|\Delta_\iota > \Delta_i / 2\}$ and $\Gu_i:=\{\iota|\Delta_\iota \leq \Delta_i / 2\}$.
    With the choice that $f(k)=1+16A^2(k+1)^2$, there is a constant $c$, such that for arbitrary $i$ with $\Delta_i > 0$, for the LCB algorithm in Alg \ref{alg:UCB_Explore_LCB_Exploit}, we have:
    \begin{align}
        \Pr(i = \piE_k) \leq 2/k^{2\alpha}+2A/k^{2\alpha-1},\quad\forall k \geq k_i := 8\alpha c\Big(\sum_{\iota \in \Gl_i} \frac{1}{\Delta_\iota^2} + \frac{4|\Gu_i|}{\Delta_i^2}\Big)\log(1+\frac{\alpha A}{\Delta_{\min}})\label{eq:def_ki}
    \end{align}
    where $c$ is the constant considered in Lem. \ref{lem:upper_bound_of_Nk_geq_k_div_scalar} (i.e. $c_f$ defined in Eq.\eqref{eq:def_c_f_new}).
\end{lemma}
\begin{proof}
    We want to remark that the constants in the definition of $k_i$ (i.e. 8 in ``$8\alpha c$'' and 4 in ``''$4|\Gu_i|$'') can be replaced by others, but we choose them carefully in order to make sure some steps in the proof of this Lemma and Thm. \ref{thm:total_regret_UCB_LCB} can go through.

    The main idea of the proof is to use Lem. \ref{lem:upper_bound_of_Nk_geq_k_div_scalar} to show that, for those arm $i$ with $\Delta_i > 0$, when $k \geq k_i$, $N_\iota(k)$ will be small for those $\iota \in \Gl_i$. As a result, there must exist an arm $j \in \Gu_i$, such that $N_j(k)$ is large than the threshold considered in Lem. \ref{lem:blessing_of_pessimism} and therefore, with high probability, arm $i$ will not be preferred.

    First, we try to apply Lem. \ref{lem:upper_bound_of_Nk_geq_k_div_scalar} to upper bound the quantity $N_j(k)$ for those arm $j\in \Gl_i$.
    For each $j\in \Gl_i$, 
    we define the following quantity, which measures the magnitude of $1/\Delta_j^2$ with $k$:
    \begin{align*}
        \gamma_{k,j} := \frac{k}{\frac{8\alpha c}{\Delta_j^2}\log(1+\frac{\alpha A}{\Delta_{\min}})}.
    \end{align*}
    We only consider $k \geq k_i$, where we always have $\gamma_{k,j} \geq 1$ based on the definition of $k_i$.

    Next, we separately consider two cases depending on whether $\gamma_{k,j} > 2A$ or not. 

    \textbf{Case 1: $\gamma_{k,j} > 2A$}: In this case, $\Delta_j$ is relatively large (or say more sub-optimal) comparing with iteration $k$. 
    For arbitrary $k \geq k_{i}$, we have
    \begin{align*}
        k =& \gamma_{k,j}\cdot \frac{8\alpha c}{\Delta_j^2 }\log(1+\frac{\alpha A}{\Delta_{\min}}) \geq 2A\cdot \frac{8\alpha c}{\Delta_j^2 }\log(1+\frac{\alpha A}{\Delta_{\min}}) \geq 2A + \frac{2\alpha c A}{\Delta_i^2} \log(1+\frac{\alpha A}{\Delta_{\min}}).
    \end{align*}
    which implies that $k$ satisfying the condition of applying Lemma \ref{lem:upper_bound_of_Nk_geq_k_div_scalar} with $\scalar = 2A$, and we can conclude that:
    $$
    \Pr(N_k(j) \geq \frac{k}{2A}) \leq \frac{2}{k^{2\alpha-1}}.
    $$

    \textbf{Case 2: $\gamma_{k,j} \leq 2A$}: Note that
    \begin{align*}
        \frac{4\alpha c}{\Delta_j^2}\log(1+\frac{\alpha A}{\Delta_{\min}}) = \frac{k}{2\gamma_{k,j}}.
    \end{align*}
    Since $2\gamma_{k,j}$ locates in the interval $[1,4A]$ and:
    \begin{align*}
        k =& \gamma_{k,j}\cdot \frac{8\alpha c}{\Delta_j^2 }\log(1+\frac{\alpha A}{\Delta_{\min}})\geq 2\gamma_{k,j} + 2\gamma_{k,j}\cdot \frac{\alpha c}{\Delta_j^2} \log(1+\frac{\alpha A}{\Delta_{\min}})
    \end{align*}
    which satisfies the condition of applying Lem. \ref{lem:upper_bound_of_Nk_geq_k_div_scalar} with $\scalar = 2\gamma_{k,j}$. Therefore, we have:
    \begin{align}
        \Pr(N_k(j)\geq \frac{4\alpha c}{\Delta_j^2}\log(1+\frac{\alpha A}{\Delta_{\min}})) = \Pr(N_k(j)\geq \frac{k}{2\gamma_{k,j}}) \leq \frac{2}{k^{2\alpha-1}}
    \end{align}
    Combining the above two cases, we can conclude that, for arbitrary $j \in \Gl_i$, 
    $$
    \Pr(N_j(k) \geq \frac{k}{2A} + \frac{4\alpha c}{\Delta_j^2}\log(1+\frac{\alpha A}{\Delta_{\min}})) \leq \min\{ \Pr(N_j(k) \geq \frac{k}{2A}), \Pr(N_j(k) \geq \frac{4\alpha c}{\Delta_j^2}\log(1+\frac{\alpha A}{\Delta_{\min}}))\} \leq \frac{2}{k^{2\alpha-1}}
    $$
    which reflects that with high probability, $\sum_{j\in \Gl_i} N_j(k)$ is small:
    \begin{align*}
        &\Pr(\sum_{j\in\Gl_i} N_j(k) \geq \frac{k}{2} + \sum_{j \in \Gl_i} \frac{4\alpha c}{\Delta_j^2}\log(1+\frac{\alpha A}{\Delta_{\min}})) \\
        \leq & \sum_{j\in\Gl_i} \Pr(N_j(k)\geq \frac{k}{2|\Gl_i|} + \frac{4\alpha c}{\Delta_j^2}\log(1+\frac{\alpha A}{\Delta_{\min}})) \tag{$P(a+b\leq c+d)\leq P(a\leq c) + P(b \leq d)$}\\
        \leq & \sum_{j\in\Gl_i} \Pr(N_j(k)\geq \frac{k}{2A} + \frac{4\alpha c}{\Delta_j^2}\log(1+\frac{\alpha A}{\Delta_{\min}})) \tag{$|\Gl_i|\leq A$}\\
        \leq& \frac{2|\Gl_i|}{k^{2\alpha-1}}\leq \frac{2A}{k^{2\alpha-1}}
    \end{align*}
    Since $\sum_{j\in \Gu_i} N_j(k) + \sum_{j\in\Gl_i} N_j(k) = k$, and note that,
    \begin{align*}
        k - (\frac{k}{2} + \sum_{j\in \Gl_i} \frac{4\alpha c}{\Delta_j^2}\log(1+\frac{\alpha A}{\Delta_{\min}})) = \frac{k}{2} - \sum_{j\in \Gl_i} \frac{4\alpha c}{\Delta_j^2}\log(1+\frac{\alpha A}{\Delta_{\min}}) = \frac{16\alpha c|\Gu_i|}{\Delta_i^2}\log(1+\frac{\alpha A}{\Delta_{\min}})
    \end{align*}
    we have:
    \begin{align*}
        &\Pr(\sum_{j\in \Gu_i} N_j(k) \leq \frac{16\alpha c|\Gu_i|}{\Delta_i^2}\log(1+\frac{\alpha A}{\Delta_{\min}})) \\
        =& \Pr(\sum_{j\in\Gl_i} N_j(k) \geq \frac{k}{2} + \sum_{j \in \Gl_i} \frac{4\alpha c}{\Delta_j^2}\log(1+\frac{\alpha A}{\Delta_{\min}})) \\
        \leq&\frac{2A}{k^{2\alpha-1}}
    \end{align*}
    Therefore, w.p, $1-\frac{2A}{k^{2\alpha-1}}$, there exists $j\in \Gu_i$, such that 
    \begin{align*}
        N_j(k) \geq \frac{1}{|\Gu_i|}\sum_{j\in \Gu_i} N_j(k) \geq \frac{16\alpha c}{\Delta_i^2}\log(1+\frac{\alpha A}{\Delta_{\min}}).
    \end{align*}
    Recall our choice of $c$ (Eq.\eqref{eq:def_c_f_new}), the above implies that:
    \begin{align}
        N_j(k) \geq \frac{32\alpha \log f(k)}{\Delta_i^2}.\label{eq:lower_bound_Gupper}
    \end{align}
    Therefore,
    \begin{align*}
        \Pr(\{i = \piE_k\}\cap\{k \geq k_{i}\})\leq & \Pr(\{i = \piE_k\}\cap\{k \geq k_{i}\}\cap\{\exists j \in \Gu_i: N_j(k) \geq \frac{32\alpha \log f(k)}{\Delta_i^2}\}) \\
        &+ \Pr(\{k \geq k_{i}\}\cap \neg \{\exists j \in \Gu_i: N_j(k) < \frac{32\alpha\log f(k)}{\Delta_i^2}\}) \\\leq & \Pr(\{i = \piE_k\}\cap\{k \geq k_{i}\}\cap\{\exists j \in \Gu_i: N_j(k) \geq \frac{32\alpha \log f(k)}{\Delta_i^2}\}) + \frac{2A}{k^{2\alpha-1}} \tag{Eq.\eqref{eq:lower_bound_Gupper}}\\
        \leq & \Pr(\{i = \piE_k\}\cap\{k \geq k_{i}\}\cap\{\exists j \in \Gu_i:  N_j(k) \geq \frac{8\alpha \log f(k)}{(\Delta_j-\Delta_j)^2}\}) + \frac{2A}{k^{2\alpha-1}} \\
        \leq & \frac{2}{k^{2\alpha}}+ \frac{2A}{k^{2\alpha-1}} \tag{Lem. \ref{lem:blessing_of_pessimism}}
    \end{align*}
\end{proof}

\begin{lemma}[Integral Lemma]\label{lem:integral_lem}
    For arbitrary $k_0 \geq 1$ and $\beta > 1$, we have:
    \begin{align*}
        \sum_{k=k_0+1}^\infty \frac{1}{k^\beta} \leq \int_{k_0}^\infty \frac{1}{x^\beta} dx \leq \frac{1}{(\beta-1)k_0^{\beta-1}} 
    \end{align*}
\end{lemma}

\ThmLCBRegret*
\begin{proof}
    Recall the definition of $k_i$ in Eq.\eqref{eq:def_ki} in Lem. \ref{lem:combining_UCB_with_LCB_new} above. For $i \geq 2$, if $\Delta_i \neq \Delta_{i-1}$, we have:
    \begin{align*}
        k_i - k_{i-1} =& \alpha c\Big(\sum_{\iota \in \Gl_i} \frac{8}{\Delta_\iota^2} - \sum_{\iota \in \Gl_{i-1}} \frac{8}{\Delta_\iota^2} + \frac{32|\Gu_i|}{\Delta_i^2} - \frac{32|\Gu_{i-1}|}{\Delta_{i-1}^2}\Big)\log(1+\frac{\alpha A}{\Delta_{\min}}) \\
        \leq & \alpha c\Big((|\Gl_i| - |\Gl_{i-1}|)\frac{32}{\Delta_{i}^2} + \frac{32|\Gu_i|}{\Delta_i^2} - \frac{32|\Gu_{i-1}|}{\Delta_{i-1}^2}\Big)\log(1+\frac{\alpha A}{\Delta_{\min}}) \tag{$\forall \iota \in \Gl_i \setminus \Gl_{i-1}$, we have $1/\Delta^2_\iota\leq 4/\Delta^2_{i}$}\\
        \leq & \alpha c\Big((|\Gu_{i-1}| - |\Gu_{i}|)\frac{32}{\Delta_{i}^2} + \frac{32|\Gu_i|}{\Delta_i^2} - \frac{32|\Gu_{i-1}|}{\Delta_{i-1}^2}\Big)\log(1+\frac{\alpha A}{\Delta_{\min}}) \tag{$|\Gu_{i-1}| + |\Gl_{i-1}|=|\Gu_{i}| + |\Gl_{i}|=A$}\\
        \leq&32\alpha c|\Gu_{i-1}|\Big(\frac{1}{\Delta_i^2}- \frac{1}{\Delta_{i-1}^2}\Big)\log(1+\frac{\alpha A}{\Delta_{\min}}).
    \end{align*}
    and if $\Delta_i = \Delta_{i-1}$, we also have:
    \begin{align*}
        k_i - k_{i-1} = 0 \leq 32\alpha c|\Gu_{i-1}|\Big(\frac{1}{\Delta_i^2}- \frac{1}{\Delta_{i-1}^2}\Big)\log(1+\frac{\alpha A}{\Delta_{\min}}).
    \end{align*}
    Moreover, for $i=1$, with the extended definition that $\Delta_0 = \infty$ (so that $1/\Delta_0^2=0$) and $|\Gu_0|=A$, we also have:
    \begin{align*}
        k_1 :=& 8\alpha c\Big(\sum_{\iota \in \Gl_1} \frac{1}{\Delta_\iota^2} + \frac{4|\Gu_1|}{\Delta_1^2}\Big)\log(1+\frac{\alpha A}{\Delta_{\min}}) \\
        \leq& 8\alpha c\Big(\sum_{\iota \in \Gl_1} \frac{4}{\Delta_1^2} + \frac{4|\Gu_1|}{\Delta_1^2}\Big)\log(1+\frac{\alpha A}{\Delta_{\min}}) \\
        =& \frac{32\alpha c A}{\Delta_1^2}\log(1+\frac{\alpha A}{\Delta_{\min}})\\
        =& 32\alpha c |\Gu_0|(\frac{1}{\Delta_1^2}-\frac{1}{\Delta_0^2})\log(1+\frac{\alpha A}{\Delta_{\min}}).
    \end{align*}
    Therefore, we have (we denote $k_A := \infty$ and $k_0:=0$):
    \begin{align*}
        \lim_{K \rightarrow \infty} \Regret_K(\algE) =& \sum_{k=1}^\infty \sum_{j:\Delta_j > 0} \Pr(j = \piE_k) \Delta_j \\
        =& \sum_{i=1}^A\sum_{k=k_{{i}-1}+1}^{k_{{i}}} \sum_{\Delta_j > 0} \Pr(j = \piE_k) \Delta_j \\
        =& \sum_{i=1}^A\sum_{k=k_{{i}-1}+1}^{k_{{i}}} \Big(\sum_{\Delta_j \geq \Delta_{i-1}} \Pr(j = \piE_k)\Delta_j + \sum_{\Delta_j < \Delta_{i-1} } \Pr(j = \piE_k) \Delta_j \Big) \\
        \leq& \sum_{i=1}^A\sum_{k=k_{{i}-1}+1}^{k_{{i}}} \Big(\sum_{\Delta_j \geq \Delta_{i-1}} \big(\frac{2}{k^{2\alpha}} + \frac{2A}{k^{2\alpha-1}}\big) \Delta_j + \sum_{\Delta_j < \Delta_{i-1} } \Pr(j = \piE_k) \Delta_j \Big) \tag{Lemma \ref{lem:combining_UCB_with_LCB_new}}\\
        \leq& \sum_{i=1}^A\sum_{k=k_{{i}-1}+1}^{k_{{i}}} \Big(\sum_{\Delta_j \geq \Delta_{i-1}} \big(\frac{2}{k^{2\alpha}} + \frac{2A}{k^{2\alpha-1}}\big) \Delta_j + \sum_{\Delta_j \leq \Delta_i } \Pr(j = \piE_k) \Delta_j \Big) \tag{$\Delta_i \leq \Delta_{i-1}$}\\
        \leq& \sum_{i=1}^A\sum_{k=k_i+1}^\infty \big(\frac{2}{k^{2\alpha}} + \frac{2A}{k^{2\alpha-1}}\big) + \sum_{i=1}^A\sum_{k=k_{{i}-1}+1}^{k_{{i}}} \sum_{\Delta_j \leq \Delta_i } \Pr(j = \piE_k) \Delta_j \\
        \leq& \sum_{i=1}^A\sum_{k=k_i + 1}^\infty \frac{2(A+1)}{k^{2\alpha-1}}  + \sum_{i=1}^A\sum_{k=k_{{i}-1}+1}^{k_{{i}}} \sum_{\Delta_j \leq \Delta_i } \Delta_i \tag{Second term is maximized when $\Pr(i=\piE_k)=1$.}\\
        \leq& \tilde{O}(\frac{A}{\alpha - 1}) + \sum_{i=1}^A \Delta_i\cdot (k_i - k_{i-1}) \tag{First term: Lemma \ref{lem:integral_lem} and some simplification; Second term: Definition of $k_{i}$.}\\
        =& \tilde{O}(\frac{A}{\alpha - 1}) + \sum_{i:\Delta_i>0} \Delta_i\cdot (k_i - k_{i-1}) \\
        \leq& \tilde{O}(\frac{A}{\alpha - 1}) + \sum_{\Delta_i > 0} 32\alpha c|\Gu_{i-1}|\Big(\frac{1}{\Delta_i}- \frac{\Delta_i}{\Delta_{i-1}^2}\Big)\log(1+\frac{\alpha A}{\Delta_{\min}}) \\
        \approx& \tilde{O}\left(\frac{A}{\alpha - 1} +\sum_{\Delta_i > 0} \alpha |\Gu_{i-1}|\Big(\frac{1}{\Delta_i}- \frac{\Delta_i}{\Delta_{i-1}^2}\Big)\right)
    \end{align*}
    According to the definition, we always have $|\Gu_{i-1}| \leq A - i + 1$, therefore,
    \begin{align*}
        \lim_{K \rightarrow \infty} \Regret_K(\algE) = \tilde{O}\left(\frac{A}{\alpha - 1} +\alpha \sum_{\Delta_i > 0} (A-i)\Big(\frac{1}{\Delta_i}- \frac{\Delta_i}{\Delta_{i-1}^2}\Big)\right)
    \end{align*}
\end{proof}

\section{Behavior Analysis of Optimistic Algorithm}
\begin{definition}[Definition of Events]\label{def:policy_agreements_events}
\begin{align*}
    &\cE_{k,h,\pi}:=\{\pi_{k,h}(s_h)\neq \pi_{h}(s_h)\},\quad \tilde{\cE}_{k,h,\pi}:=\cE_{k,h,\pi}\cap \bigcap_{\ph=1}^h \cE_{k,\ph-1,\pi}^\complement,\quad \bar{\cE}_{k,\pi}:=\bigcup_{h=1}^H \cE_{k,h,\pi},\\
    &\cE_{k,h}:=\{\pi_{k,h}(s_h)\not\in \Pi^*_{h}(s_h)\},\quad \tilde{\cE}_{k,h,}:=\cE_{k,h}\cap \bigcap_{\ph=1}^h \cE_{k,\ph-1}^\complement,\quad \bar{\cE}_k:=\bigcup_{h=1}^H \cE_{k,h}\\
\end{align*}
\end{definition}
In another word, $\cE_{k,h,\pi}$ means $\pi_k$ disagrees with $\pi$ at state $s_h$ which occurs at step $h$, $\tilde{\cE}_{k,h,\pi}$ means the first disagreement between $\pi_k$ and $\pi$ occurs at step $h$, and $\bar{\cE}_{k,\pi}$ denotes the event that there exists one state $s_h$ at some time step $h\in[H]$ such that $\pi_k$ agrees with $\pi_k$ at $s_h$.

Besides, $\cE_{k,h}$ denotes the events that $\pi_{k,h}(s_h)$ will not be taken by any optimal policy. Note that here we use $\Pi^*_h(s_h)$ to denote the set of all possible optimal actions at state $s_h$. Given a deterministic optimal policy $\pi^*$, in general $\cE_{k,h} \neq \cE_{k,h,\pi^*}$ when there are multiple optimal actions at one state.

\begin{lemma}\label{lem:value_decomposition}
    For arbitrary reward function $R$, given a fixed deterministic policy $\pi$, we have:
    \begin{align*}
        V^{\pi}_1(s_1) - V^{\pi_k}_1(s_1) = &\EE_{\pi_k}[\sum_{h=1}^H \mathbb{I}[\tilde{\cE}_{k,h,\pi}](V^{\pi}_h(s_h)-V^{\pi_k}_h(s_h))]
    \end{align*}
\end{lemma}
\begin{proof}
    \begin{align*}
        V^{\pi}_1(s_1) - V^{\pi_k}_1(s_1) =& \mathbb{I}[{\cE}_{k,1,\pi}^\complement] \Big(Q^{\pi}_1(s_1,\pi_k)-Q^{\pi_k}_1(s_1,\pi_k)\Big)+\mathbb{I}[{\cE}_{k,1,\pi}] \Big(V^\pi_1(s_1)-V^{\pi_k}_1(s_1)\Big)\\
        =& \EE_{\pi_k}[\mathbb{I}[{\cE}_{k,1,\pi}^\complement]\Big(V^{\pi}_2(s_2)-V^{\pi_k}_2(s_2)\Big)]+\mathbb{I}[\tilde{\cE}_{k,1,\pi}] \Big(V^\pi_1(s_1)-V^{\pi_k}_1(s_1)\Big)\tag{$\tilde{\cE}_{k,1,\pi}=\cE_{k,1,\pi}$ by definition}\\
        =& \EE_{\pi_k}[\mathbb{I}[{\cE}_{k,2,\pi}^\complement\cap {\cE}_{k,1,\pi}^\complement]\Big(Q^{\pi}_2(s_2,\pi_k)-Q^{\pi_k}_2(s_2,\pi_k)\Big)]\\
        &+\EE_{\pi_k}[\mathbb{I}[{\cE}_{k,2,\pi}\cap {\cE}_{k,1,\pi}^\complement]\Big(V^{\pi}_2(s_2)-V^{\pi_k}_2(s_2)\Big)]+\mathbb{I}[\tilde{\cE}_{k,1,\pi}] \Big(V^\pi_1(s_1)-V^{\pi_k}_1(s_1)\Big)\\
        =& \EE_{\pi_k}[\mathbb{I}[{\cE}_{k,2,\pi}^\complement\cap {\cE}_{k,1,\pi}^\complement]\Big(Q^{\pi}_2(s_2,\pi_k)-Q^{\pi_k}_2(s_2,\pi_k)\Big)]\\
        &+\EE_{\pi_k}[\mathbb{I}[\tilde{\cE}_{k,2,\pi}]\Big(V^{\pi}_2(s_2)-V^{\pi_k}_2(s_2)\Big)]+\mathbb{I}[\tilde{\cE}_{k,1,\pi}] \Big(V^\pi_1(s_1)-V^{\pi_k}_1(s_1)\Big)\\
        =&...\\
        =&\EE_{\pi_k}[\sum_{h=1}^H \mathbb{I}[\tilde{\cE}_{k,h,\pi}](V^{\pi}_h(s_h)-V^{\pi_k}_h(s_h))]
    \end{align*}
\end{proof}
\begin{lemma}[Relationship between Density Difference and Policy Disagreement Probability]\label{lem:relationship_between_occupancy_and_failure_rate}
    \begin{align*}
        d^{\pi_k}(s_h,a_h) \geq d^{\pi}(s_h,a_h) - \min\{\Pr(\bar{\cE}_{k,\pi}|\pi_k),~d^\pi(s_h,a_h)\},\quad \forall s_h\in\cS_h,a_h\in\cA_h,h\in[H]
    \end{align*}
    where we use $\Pr(\bar{\cE}_{k,\pi}|\pi_k)$ as a short note of $\EE_{s_1,a_1,s_2,a_2...,s_H,a_H \sim \pi_k}[\bar{\cE}_{k,\pi}]$.
\end{lemma}
\begin{proof}
    By applying Lemma \ref{lem:value_decomposition} with $\delta_{s_h,a_h}:=\mathbb{I}[S_h=s_h,A_h=a_h]$ as reward function, we have:
    \begin{align*}
        d^{\pi}(s_h,a_h) - d^{\pi_k}(s_h,a_h) = & V^{\pi}_1(s_1;\delta_{s_h,a_h}) - V^{\pi_k}_1(s_1;\delta_{s_h,a_h}) \\
        =&\EE_{\pi_k}[\sum_{\ph=1}^h \mathbb{I}[\tilde{\cE}_{k,\ph,\pi}](V^{\pi}_\ph(s_\ph;\delta_{s_h,a_h})-V^{\pi_k}_\ph(s_\ph;\delta_{s_h,a_h}))] \tag{$V_\ph^\pi = V_\ph^{\pi_k} =0$ for all $\ph \geq h+1$}\\
        \leq& \EE_{\pi_k}[\sum_{\ph=1}^h \mathbb{I}[\tilde{\cE}_{k,\ph,\pi}] V^{\pi}_\ph(s_\ph;\delta_{s_h,a_h})]\tag{$V^{\pi_k}_{h'}(s_h';\delta_{s_h,a_h}) \geq 0$} \\
        \leq& \EE_{\pi_k}[\sum_{\ph=1}^h \mathbb{I}[\tilde{\cE}_{k,\ph,\pi}]] \tag{$V^{\pi}_{h'}(s_h';\delta_{s_h,a_h}) \leq 1$} \\
        \leq& \EE_{s_1,a_1,s_2,a_2...,s_H,a_H \sim \pi_k}[\bar{\cE}_{k,\pi}] = \Pr(\bar{\cE}_{k,\pi}|\pi_k)
    \end{align*}
    which implies that,
    \begin{align*}
        d^{\pi_k}(s_h,a_h) \geq d^{\pi}(s_h,a_h) - \Pr(\bar{\cE}_{k,\pi}|\pi_k)
    \end{align*}
    Combining with $d^{\pi_k} \geq 0$, we finish the proof.
\end{proof}

\begin{definition}[Conversion to Optimal Deterministic Policy]\label{def:conversion_to_opt_det_policy}
    Given arbitrary deterministic policy $\pi=\{\pi_1,...,\pi_H\}$, we use $\Pi^*\circ \pi=\{\pi_1^*,...,\pi^*_H\}$ to denote an optimal deterministic policy, such that:
    \begin{align*}
        \pi^*_{h}(s_h) = \begin{cases}
            \pi_{h}(s_h),\quad & \text{if}~\pi_{h}(s_h)\in \Pi^*_h(s_h);\\
            \text{Select}(\Pi^*_h(s_h)),\quad &\text{otherwise}.
        \end{cases}
    \end{align*}
    where $\text{Select}$ is a function which returns the first optimal action from $\Pi^*_h(s_h)$.
\end{definition}
In another word, $\Pi^*\circ\pi$ agrees with $\pi$ if $\pi_{h}(s_h)$ is one of the optimal action at state $s_h$. Otherwise, $\Pi^*\circ\pi$ takes one of a fixed optimal action from $\Pi^*_h(s_h)$. In order to make sure $\Pi^*\circ$ is a deterministic mapping, we assume function $Select$ only choose the first optimal action in $\Pi^*(s_h)$ (ordered by index of action).

\ThmORegretPDensity*
\begin{proof}
    For each $\pi_k$, we construct an optimal deterministic policy $\pi^*_k:=\Pi^*\circ\pi_k$, where $\Pi^*\circ$ is defined in Def. \ref{def:conversion_to_opt_det_policy}.
    By applying Lemma \ref{lem:value_decomposition} with the reward function in MDP, and $\pi=\pi_k^*$, we have:
    \begin{align*}
        V^{\pi^*_k}_1(s_1) - V^{\pi_k}_1(s_1) = &\EE_{\pi_k}[\sum_{h=1}^H \mathbb{I}[\tilde{\cE}_{k,h,\pi_k^*}](V^{\pi^*_k}_h(s_h)-V^{\pi_k}_h(s_h))] \\
        \geq& \EE_{\pi_k}[\sum_{h=1}^H \mathbb{I}[\tilde{\cE}_{k,h,\pi_k^*}](V^{\pi^*_k}_h(s_h)-Q^{\pi^*_k}_h(s_h, \pi_k(s_h)))]\\
        \geq& \EE_{\pi_k}[\sum_{h=1}^H \mathbb{I}[\tilde{\cE}_{k,h,\pi_k^*}]\Delta_{\min}]=\Delta_{\min} \Pr(\bar{\cE}_{k,\pi_k^*}|\pi_k)
    \end{align*}
    Therefore, we have:
    \begin{align*}
        \Pr(\bar{\cE}_{k,\pi_k^*}|\pi_k) \leq \frac{1}{\Delta_{\min}}(V^{\pi^*_k}_1(s_1)-V^{\pi_k}_1(s_1))
    \end{align*}
    By applying Lemma \ref{lem:relationship_between_occupancy_and_failure_rate}, we have:
    \begin{align*}
        d^{\pi_k}(s_h,a_h) \geq d^{\pi_k^*}(s_h,a_h) - \frac{1}{\Delta_{\min}} \Big(V^*_1(s_1)-V^{\pi_k}_1(s_1)\Big),\quad\forall s_h\in\cS_h,a_h\in\cA_h,h\in[H]
    \end{align*}
    After the same discussion for all $k\in[K]$, and the above inequality of each $k$ together, we have:
    \begin{align*}
        \sum_{k=1}^K d^{\pi_k}(s_h,a_h) \geq \sum_{k=1}^K d^{\pi^*_k}(s_h,a_h) - \frac{1}{\Delta_{\min}} (\sum_{k=1}^K V^*_1(s_1)-V^{\pi_k}_1(s_1)\Big)
    \end{align*}
\end{proof}


\begin{corollary}[Unique Optimal Policy]\label{corl:unique_optimal_policy}
    When $|\Pi^*|=1$, Thm. \ref{thm:algO_regret_vs_algP_density} implies that:
    \begin{align*}
        \sum_{k=1}^K d^{\pi_k}(s_h,a_h) \geq K d^{\pi^*}(s_h,a_h) - \frac{1}{\Delta_{\min}} (\sum_{k=1}^K V^*_1(s_1)-V^{\pi_k}_1(s_1)\Big)
    \end{align*}
\end{corollary}

\ThmExistWellCovered*
\begin{proof}
    For arbitrary $h\in [H]$, we define: 
    \begin{align*}
        N_{I_K^*}(s_h,a_h):=\sum_{k=1}^K \mathbb{I}[d^{\pi^*_k}(s_h,a_h) > 0]
    \end{align*}
    In another word, $N_{I_K^*}(\cdot,\cdot)$ denotes the number of optimal policies in the sequence, which can hit state $s_h$ and take action $a_h$. 
    
    Next, we define $\cZ_h^*$, $\cZ_{h}^{\text{insuff}}$ and $\Pi_{h,\text{insuff}}$ as
    \begin{align*}
        \cZ_h^*:=&\{(s_h,a_h)\in\cS_h,\cA_h|\exists \pi^*\in\Pi^*,~s.t.~d^{\pi^*}(s_h,a_h) > 0\}\\
        \cZ_{h}^{\text{insuff}}:= &\{(s_h,a_h)|(s_h,a_h)\in\cZ^*_h:~N_{I_K^*}(s_h,a_h)<\frac{K}{2(|\cZ^*_{h,\text{div}}|+1)H}\}\\
        I_h^{\text{insuff}}:= &\{k\in[K]:~\exists (s_h,a_h)\in\cZ_{h}^{\text{insuff}},~s.t.~d^{\pi_k}(s_h,a_h) > 0\}
    \end{align*}
    In a word, $\cZ_h^*$ is the collection of states actions reachable by at least one optimal policy, $\cZ_{h}^{\text{insuff}}$ is a collection of ``insufficiently hitted'' states actions at step $h$, which are only covered by a small portion of optimal policies in the sequence, and $I_h^{\text{insuff}}$ is a collection of the index of the optimal policies in the sequence, which cover at least one state action pair in $\cZ_{h}^{\text{insuff}}$. 

    
    Note that we must have $\cZ_{h}^{\text{insuff}}\subset \cZ_{h,\text{div}})$, because if one state action pair $s_h,a_h$ is reachable by arbitrary deterministic policy, then $N_{I^*_k}(s_h,a_h) = K$. Then, we have:
    \begin{align*}
        |I_h^{\text{insuff}}| < |\cZ_{h}^{\text{insuff}}|\cdot \frac{K}{2(|\cZ_{h,\text{div}}|+1) H} \leq |\cZ_{h,\text{div}}|\cdot \frac{K}{2(|\cZ_{h,\text{div}}|+1) H} \leq \frac{K}{2H}
    \end{align*}
    We define $I_{1:H}^{\text{suff}}:=I_K \setminus \bigcup_{h=1}^H I_h^{\text{insuff}}$. Intuitively, $I_{1:H}^{\text{suff}}$ is the set including the indices of optimal policies in the sequence only hitting those states which are covered by most of the other optimal policies. In fact, $I_{1:H}^{\text{suff}}$ is non-empty since:
    \begin{align*}
         |I_{1:H}^{\text{suff}}|\geq K - \frac{K}{2H} \cdot H = \frac{K}{2}
    \end{align*}
    We use $\pi^*_{I_{1:H}^{\text{suff}}}$ to denote the average mixture policy over $\{\pi_i^*: i\in I_{1:H}^{\text{suff}}\}$, a direct result is that:
    \begin{align*}
        \sum_{k=1}^K d^{\pi^*_k}(s_h,a_h) \geq \sum_{k\in I_{1:H}^{\text{suff}}} d^{\pi^*_k}(s_h,a_h) = |I_{1:H}^{\text{suff}}|\cdot d^{\pi^*_{I_{1:H}^{\text{suff}}}} \geq \frac{K}{2}\cdot d^{\pi^*_{I_{1:H}^{\text{suff}}}}
    \end{align*}
    On the other hand, for all $s_h,a_h$ such that $d^{\pi^*_{I_{1:H}^{\text{suff}}}}(s_h,a_h) > 0$, we must have $(s_h,a_h)\not\in \cZ_{h}^{\text{insuff}}$, and therefore:
    \begin{align*}
        \sum_{k=1}^K d^{\pi^*_k}(s_h,a_h) \geq \frac{K}{2(|\cZ_{h,\text{div}}|+1)H}d^*_{h,\min}(s_h,a_h)
    \end{align*}
    Combining the above two inequalities, we finish the proof.
\end{proof}

\section{Analysis of Pessimistic Value Iteration}\label{appx:analysis_of_PVI}

In this section, we provide analysis for Alg. \ref{alg:PVI}. Our analyses base on an extension of the Clipping Trick in \citep{simchowitz2019non} into our setting. 

\subsection{Underestimation and Some Concrete Choices of Bonus Term}\label{appx:choice_of_bonus_term}
\begin{lemma}[Underestimation]\label{lem:underest_PVI}
    Given a $\textbf{Bonus}$ satisfying Cond. \ref{cond:bonus_term}, for arbitrary dataset $D_k$ consisting of $k$ trajectories by a sequence of policies $\pi_1,...,\pi_k$, by running Alg \ref{alg:PVI} with $D_k$ and the bonus term $b(\cdot,\cdot)$ returned by $\textbf{Bonus}(D_k,\delta_k)$, on the events $\cEB$ defined in Cond. \ref{cond:bonus_term}:
    \begin{align}
        \forall h\in[H],\forall s_h\in\cS_h,a_h\in\cA_h,\quad \hQ_h(s_h,a_h) \leq Q^{\pi_{\hQ,h}}(s_h,a_h) \leq Q^{*}(s_h,a_h) \label{eq:under_estimation}
    \end{align}
    where we use $\pi_\hQ=\{\pi_{\hQ,1},...,\pi_{\hQ,H}\}$ to denote the greedy policy w.r.t. $\hQ$.
\end{lemma}
\begin{proof}
    We only prove the first inequality holds, since the second one holds directly because of the definition of optimal policy.

    First of all, $V_{H+1} = 0 \leq V^{\pi_\hQ}_{H+1}$ holds directly, which implies that Eq.\eqref{eq:under_estimation} holds at step $h=H$ as a result of the deterministic reward function. 
    
    Now, we conduct the induction. Suppose Eq.\eqref{eq:under_estimation} already holds for $h+1$, which implies that:
    \begin{align}
        \hV_{h+1}(s_{h+1}) = \hQ_{h+1}(s_{h+1},\pi_{\hQ, h+1}(s_{h+1})) \leq Q^{\pi_\hQ}_{h+1}(s_{h+1},\pi_{\hQ, h+1}(s_{h+1})) = V^{\pi_\hQ}_{h+1}(s_{h+1})\label{eq:under_est_mid_step}
    \end{align}
    then, at step $h$, we have:
    \begin{align*}
        Q_h(s_h,a_h) - Q^{\pi_\hQ}(s_h,a_h) =& \hat P_h \hV_{h+1} (s_h,a_h) - b_h(s_h,a_h) - P_hV^\pi_{h+1}(s_h,a_h)\\
        =& \underbrace{(\hat P_h - P_h)\hV_{h+1}(s_h,a_h) - b_h(s_h,a_h)}_{\text{part 1}} + \underbrace{P_h(\hV_{h+1}-V_{h+1}^{\pi_\hQ})(s_h,a_h)}_{\text{part 2}}
    \end{align*}
    As we can see, part 1 is non-positive with probability $1-\delta$ as a result of Cond. \ref{cond:bonus_term}, while part 2 is also less than or equal to zero because of the induction condition in Eq.\eqref{eq:under_est_mid_step}.
\end{proof}

\paragraph{Choice 1: Naive Bound}
According to Hoeffding inequality, with probability $1-\delta/(SAH)$, we have the following holds for each $s_h,a_h,h$
\begin{align*}
    |\hP_h V_{h+1} - P_h V_{h+1}| \leq \|\hP_h - P_h\|_1 \|V_{h+1}\|_\infty \leq \|\hP_h - P_h\|_1H \leq c_1 HS \sqrt{\frac{\log(SAH/\delta)}{N(s_h,a_h)}}
\end{align*}
which implies that condition \ref{cond:bonus_term} holds with probability $1-\delta$ as long as:
\begin{align*}
    b_h(N, \delta) := HS \sqrt{\frac{\log(SAH/\delta)}{N}}
\end{align*}

\paragraph{Choice 2: Adaptive Bonus Term based on the ``Bernstein Trick''}
One can also consider an analogue of the bonus term functions in Alg. 3 of \citep{simchowitz2019non}, which is originally designed for optimistic algorithms. We omit the discussions here.

\subsection{Definition of ``Surplus'' in Pessimistic Algorithms and the Clipping Trick}
We consider the pessimistic algorithm, and denote the estimation of value function as $\hQ$, $\hV$. We assume they are pessimistic estimation, i.e.:
\begin{align*}
    V_h^*(s_h) = Q^*_h(s_h,\pi^*) \geq Q^*_h(s_h,\pi_\hQ) \geq V^{\pi_k}_h(s_h) \geq \hQ_h(s_h, \pi_\hQ) \geq \hQ_h(s_h,\pi^*).
\end{align*}
\begin{definition}[Definition of Surplus in Pessimistic Algorithm setting]\label{def:deficit}
We define the surplus in Pessimistic Algorithm setting:
\begin{align*}
    \surplus_{k,h}(s_h,a_h) =& r(s_h,a_h) + \mP_h\hV_{k,h+1}(s_h,a_h) - \hQ_{k,h}(s_h,a_h).
\end{align*}
\end{definition}
Because of the underestimation, different from the surplus in overestimation cases \citep{simchowitz2019non}, here we flip the role between $\hQ$ and $r+\mP\hV$ to make sure the quantity is non-negative (with high probability).

Based on our definition, we have the following lemma:
\begin{lemma}\label{lem:relationship_between_value_diff_and_deficit}
    Under the same condition as Lemma \ref{lem:underest_PVI}, for arbitrary $h,s_h$, the policy $\piPVI_k$ returned by Alg.\ref{alg:PVI} satisfying:
    \begin{align*}
        V^{\piPVI_k}_h(s_h) - \hV_{k,h}(s_h) =& \EE_{\piPVI_k}[\sum_{\ph=h}^H\surplus_{k,\ph}(s_\ph,a_\ph)|s_h]
    \end{align*}
    Moreover, for arbitrary optimal deterministic or non-deterministic policy $\pi^*$, we have:
    \begin{align*}
        V^*_h(s_h) - \hV_{k,h}(s_h) \leq& V_h^*(s_h) - \hQ_{k,h}(s_h,\pi^*) \leq \EE_{\pi^*}[\sum_{\ph=h}^H \surplus_{k,\ph}(s_\ph,a_\ph)]
    \end{align*}
\end{lemma}
\begin{proof}
\begin{align*}
    &V^{\piPVI_k}_h(s_h) - \hV_{k,h}(s_h) \\
    =& \EE_{a_h\sim \piPVI_k}[r(s_h,a_h)+\mP_h V^{\piPVI_k}_{h+1}(s_h,a_h)-\hQ_{k,h}(s_h,a_h)\pm \mP_h \hV_{k,h+1}(s_h,a_h)]\\
    =&\EE_{\piPVI_k}[r(s_h,a_h)+\mP_h\hV_{k,h+1}(s_h,a_h)-\hQ_{k,h}(s_h,a_h)+\mP_h (V_{h+1}^{\piPVI_k}-\hV_{k,h+1})(s_h,a_h)]\\
    =&\EE_{\piPVI_k}[\sum_{\ph=h}^H\surplus_{k,\ph}(s_\ph,a_\ph)|s_h]
\end{align*}
Besides, given arbitrary optimal policy $\pi^*$, we have:
\begin{align*}
    &V_h^*(s_h) - \hV_{k,h}(s_h) \\
    =& V_h^*(s_h) - \hQ_{k,h}(s_h,\piPVI_k) \leq V_h^*(s_h) - \hQ_{k,h}(s_h,\pi^*) \tag{$\piPVI_k$ is greedy policy w.r.t. $\hQ_k$}\\
    =& \EE_{a_h\sim \pi^*}[r(s_h,a_h)+\mP_h\hV_{k,h+1}(s_h,a_h)-\hQ_{k,h}(s_h,a_h)+\mP_h(V_{h+1}^*-\hV_{k,h+1})(s_h,a_h)] \\
    \leq& \EE_{\pi^*}[\sum_{\ph=h}^H \surplus_{k,\ph}(s_\ph,a_\ph)]
\end{align*}
\end{proof}

\begin{lemma}\label{lem:upper_bound_surplus}
    Under the same condition as Lemma \ref{lem:underest_PVI}, we have:
    \begin{align*}
        \surplus_{k,h} \leq \min\{H-h+1, 2B_1\sqrt{\frac{\log (B_2/\delta_k)}{N_{k,h}(s_h,a_h)}}\}.
    \end{align*}
\end{lemma}
\begin{proof}
    \begin{align*}
        \surplus_{k,h} :=& r(s_h,a_h) + \mP_h\hV_{k,h+1}(s_h,a_h) - \hQ_{k,h}(s_h,a_h)\\
        =& \mP_h\hV_{k,h+1} - \hmP_{k,h}\hV_{k,h+1} + b_{k,h}(s_h,a_h) \leq 2b_{k,h}(s_h,a_h) \leq 2B_1\sqrt{\frac{\log (B_2/\delta_k)}{N_{k,h}(s_h,a_h)}}.
    \end{align*}
    On the other hand, because the reward function is always locates in $[0,1]$ and $\hQ$ is always larger than zero, we have $\surplus_{k,h}(s_h,a_h) \leq H-h+1 \leq H$.
\end{proof}


In the following, we define
\begin{align*}
    \ddot{\surplus}_{k,h}(s_h,a_h) := \text{clip}[\surplus_{k,h}(s_h,a_h)|\epsClip].
\end{align*}
where $\epsClip := \frac{\Delta_{\min}}{2H+2}$, and $\Clip[x|\epsilon] := x\cdot \mathbb{I}[x\geq\epsilon]$.
Then, we recursively define 
\begin{align*}
    \ddot{Q}^{\pi}_{k,h}(s_h,a_h)=\EE_{\pi_h}[r(s_h,a_h)-\ddot{\surplus}_{k,h}(s_h,a_h)+\mP_h\ddot{V}^{\pi}_{k,h+1}(s_h,a_h)|s_h,a_h],\quad \ddot{V}^{\pi}_{k,h}(s_h):= \ddot{Q}^{\pi}_{k,h}(s_h,\pi_{h})
\end{align*}
Note that although different optimal policies $\pi^*$ and $\tilde{\pi}^*$ have the same optimal value $V^*$, $\ddot{V}^{\pi^*}$ may no longer equal to $\ddot{V}^{\tilde{\pi}^*}$ because they may have different state occupancy and $\ddot{V}$ depends on $\surplus$. Therefore, in the following, when we consider the $\ddot{V}$ for optimal policies, we will always specify which optimal policy we are referring to.
\begin{lemma}[Relationship between $\ddot{V}^{\pi^*}$, $V^{\piPVI_k}$ and $\hV_{k,h}$]\label{lem:Vddot_and_V}
    Under the same condition as Lemma \ref{lem:underest_PVI}, for arbitrary optimal policy $\pi^*$, we have:
    \begin{align*}
        \ddot{V}^{\pi^*}_{k,h}(s_h) \leq \hV_{k,h}(s_h) + (H-h+1)\epsClip \leq V^{\piPVI_k}(s_h) + (H-h+1)\epsClip
    \end{align*}
\end{lemma}
\begin{proof}
    Note that:
    \begin{align*}
        \ddot{\surplus}_{k,h}(s_h,a_h)\geq \surplus_{k,h}(s_h,a_h) - \epsClip
    \end{align*}
    Therefore,
    \begin{align*}
        V^*_{h}(s_h) - \ddot{V}^{\pi^*}_{h}(s_h) =& \EE_{\pi^*}[\sum_{h=\ph}^H \ddot{\surplus}_{k,\ph}(s_\ph,a_\ph)|s_h]\\
        \geq& \EE_{\pi^*}[\sum_{\ph=h}^H {\surplus}_{k,\ph}(s_\ph,a_\ph)-\epsClip|s_h]\\
        \geq&V^*_{h}(s_h) - \min\{\hQ_{k,h}(s_h,\pi^*), \hV_{k,h}(s_h)\} - (H-h+1)\epsClip \tag{Lemma \ref{lem:relationship_between_value_diff_and_deficit}} \\
        \geq&V^*_{h}(s_h) - \min\{Q^{\piPVI_k}_h(s_h,\pi^*), V^{\piPVI_k}_h(s_h)\} - (H-h+1)\epsClip \tag{Underestimation (Lemma \ref{lem:underest_PVI})}
    \end{align*}
    Therefore,
    \begin{align*}
        \ddot{V}^{\pi^*}_{k,h}(s_h) \leq \hV_{k,h}(s_h) + (H-h+1)\epsClip \leq V^{\piPVI_k}_h(s_h) + (H-h+1)\epsClip
    \end{align*}
\end{proof}

\subsection{Additional Lemma for the Analysis of the Regret of $\algE$ when Optimal Deterministic Policies are non-unique}
We first introduce a useful Lemma related to the clipping operator from \citep{simchowitz2019non}
\begin{lemma}[Lemma B.3 in \citep{simchowitz2019non}]\label{lem:property_of_clip_operator}
    Let $M \geq 2$, $a_1,...a_m \geq 0$ and $\epsilon \geq 0$. $\Clip[\sum_{i=1}^m a_i | \epsilon ] \leq 2\sum_{i=1}^m \Clip [a_i|\frac{\epsilon}{2m}]$.
\end{lemma}
Next, based on definition of $d_{\min}$ in Eq.\eqref{eq:def_of_d_min}, we have the following Lemma:

\begin{lemma}\label{lem:minimal_gap_for_det_subopt_policy}
    Given arbitrary deterministic policy $\pi$, if $\pi\not\in\Pi^*$, we have:
    \begin{align*}
        V^*_1(s_1) - V^\pi_1(s_1) \geq d_{\min} \Delta_{\min}
    \end{align*}
\end{lemma}
\begin{proof}
    We use $\pi^*:=\Pi^*\circ\pi$ to denote the converted deterministic optimal policy, where $\Pi^*\circ$ is defined in Def. \ref{def:conversion_to_opt_det_policy}. As a direct application of Lemma \ref{lem:value_decomposition}, we have:
    \begin{align*}
        V^*_1(s_1) - V^\pi_1(s_1) =& V^{\pi^*}_1(s_1) - V^\pi_1(s_1) = \EE[\sum_{h=1}^H \mathbb{I}[\tilde{\cE}_h] (V^{\pi^*}_h(s_h) - V^\pi_h(s_h))]\\
        \geq&\Delta_{\min} \EE[\sum_{h=1}^H \mathbb{I}[\tilde{\cE}_h]] \\
        \geq& \Delta_{\min} \Pr(\tilde{\cE}_{h_{init}})\\
        \geq& \Delta_{\min} d_{\min}
    \end{align*}
    where we use $\tilde{\cE}_h$ to denote the event that at step $h$, $\pi$ first disagrees with $\pi^*$, or equivalently, $\pi$ first take non-optimal action; in the second inequality, we define $h_{init}:= \min_{h\in[H]},\quad s.t.\quad \Pr(\tilde{\cE}_h) > 0$. Besides, the last inequality is because:
    \begin{align*}
        \Pr(\tilde{\cE}_{h_{init}}) =& \sum_{s_{h_{init}} \in \cS_{h_{init}}} \mathbb{I}[\pi^*_h(s_h) \neq \pi_h(s_h)] d^\pi(s_h)\\
        =& \sum_{s_{h_{init}} \in \cS_{h_{init}}} \mathbb{I}[\pi^*_h(s_h) \neq \pi_h(s_h)] d^{\pi^*}(s_h)\\
        \geq& \sum_{s_{h_{init}} \in \cS_{h_{init}}} \mathbb{I}[\pi^*_h(s_h) \neq \pi_h(s_h)] d_{\min}\\
        \geq& d_{\min}
    \end{align*}
    where the last step is because, according to the definition of $h_{init}$, there is at least one $s_h\in\cS_h$ such that $\mathbb{I}[\pi^*_h(s_h) \neq \pi_h(s_h)]=1$.
\end{proof}

\subsection{Upper Bound for the Regret of $\algE$}
\ThmClipTrick*
\begin{proof}
    We separately discuss the cases when there are unique or multiple deterministic optimal policies.
    \paragraph{Case 1: Unique Deterministic Optimal Policy}
    For arbitrary $h,s_h$, suppose $\piPVI_k(s_h)\not\in\Pi^*(s_h)$, we have:
    \begin{align*}
        V^*_h(s_h) - \ddot{V}^{\pi^*}_{k,h}(s_h) \geq& V_h^*(s_h) - \hV_{k,h}(s_h) - (H-h+1)\epsClip \\
        \geq & \frac{1}{2}\Big(V_h^*(s_h) - \hV_{k,h}(s_h)\Big) + \frac{1}{2}\Big(V_h^*(s_h) - V_h^{\piPVI_k}(s_h)\Big) - \frac{\Delta_{\min}}{2}\\
        \geq & \frac{1}{2}\Big(V_h^*(s_h) - \hV_{k,h}(s_h)\Big) + \frac{1}{2}\Big(V_h^*(s_h) - Q_h^*(s_h,\piPVI_k)\Big) - \frac{\Delta_{\min}}{2}\\
        = & \frac{1}{2}\Big(V_h^*(s_h) - \hV_{k,h}(s_h)\Big) + \frac{\Delta_h(s_h,\piPVI_k(s_h))}{2} - \frac{\Delta_{\min}}{2} \\
        \geq& \frac{1}{2}\Big(V_h^*(s_h) - \hV_{k,h}(s_h)\Big)
    \end{align*}
    Recall the definition of Events in Def.\ref{def:policy_agreements_events}, and note that when the optimal policy is unique, the events $\tilde{\cE}_{k,h,\pi^*}, \bar{\cE}_{k,h,\pi^*}$ collapse to $\tilde{\cE}_{k,h}, \bar{\cE}_{k,h}$, respectively. For arbitrary optimal policy $\pi^*$, we have:
    \begin{align*}
        V^*_1(s_1) - \ddot{V}^{\pi^*}_1(s_1) =& V^{\pi^*}_1(s_1) - \ddot{V}^{\pi^*}_1(s_1)\\
        =& \mathbb{I}[\cE_{k,1}]\Big(V^*_1(s_1) - \ddot{V}_{1}^{\pi^*}(s_1)\Big) + \mathbb{I}[\cE_{k,1}^\complement]\Big(V^*_1(s_1) - \ddot{V}_{1}^{\pi^*}(s_1)\Big)\\
        \geq& \mathbb{I}[\cE_{k,1}]\frac{1}{2}\Big(V^*_1(s_1) - \hV_{k,1}(s_1)\Big)+ \mathbb{I}[\cE_{k,1}^\complement]\mP (V^*_{2}-\ddot{V}_{2}^{\pi^*})(s_1,\pi^*)\\
        \geq& ...\\
        \geq&\frac{1}{2} \EE_{\pi^*}[\sum_{h=1}^H \mathbb{I}[\tilde{\cE}_{k,h}] (V_h^*(s_h) - \hV_{k,h}(s_h))]
    \end{align*}
    Besides, on the other hand,
    \begin{align*}
        V_1^*(s_1)-V^{\piPVI_k}_1(s_1) 
        =& \mathbb{I}[\cE_{k,1}](V_1^{\pi^*} - V^{\piPVI_k}_1(s_1)) + \mathbb{I}[\cE_{k,1}^\complement](V_1^{\pi^*} - Q^{\piPVI_k}_1(s_1,\pi^*)) \\
        =& \mathbb{I}[\cE_{k,1}](V_1^{\pi^*} - V^{\piPVI_k}_1(s_1)) + \mathbb{I}[\cE_{k,1}^\complement]\mP_1(V_2^* - V^{\piPVI_k}_2)(s_1,\pi^*)) \\
        =& ... \\
        =& \EE_{\piPVI_k}[\sum_{h=1}^H \mathbb{I}[\tilde{\cE}_{k,h}] (V_h^{\pi^*}(s_h) - V^{\piPVI_k}_h(s_h))]\\
        \leq & \EE_{\piPVI_k}[\sum_{h=1}^H \mathbb{I}[\tilde{\cE}_{k,h}] (V_h^{\pi^*}(s_h) - \hV_{k,h}(s_h))]
    \end{align*}
    Combining the above two results and Lemma \ref{lem:upper_bound_surplus}, we finish the discussion for Case 1.

    
    \paragraph{Case 2: Non-unique Optimal Deterministic Policies}
    From Lemma \ref{lem:relationship_between_value_diff_and_deficit}, we know that,
    \begin{align*}
        V^*_1(s_1) - V^{\piPVI_k}_1(s_1) \leq V^*_1(s_1) - \hV_{k,1}(s_1) \leq& \EE_{\pi^*}[\sum_{h=1}^H \surplus_{k,h}(s_h,a_h)]
    \end{align*}
    where $\pi^*$ can be arbitrary optimal policy. Combining with Lemma \ref{lem:minimal_gap_for_det_subopt_policy}, we know that:
    \begin{align*}
        V^*_1(s_1) - V^{\piPVI_k}_1(s_1) \leq& \Clip[\EE_{\pi^*}[\sum_{h=1}^H \surplus_{k,h}(s_h,a_h)]|d_{\min}\Delta_{\min}]\\
        \leq& 2\sum_{h=1}^H \sum_{s_h\in\cS_h,a_h\in\cA_h} \Clip[d^{\pi^*}(s_h,a_h)\surplus_{k,h}(s_h,a_h)|\frac{d_{\min}\Delta_{\min}}{2SAH}] \tag{Lemma \ref{lem:property_of_clip_operator}}\\
        \leq& 2\sum_{h=1}^H \EE_{\pi^*}[\Clip[\surplus_{k,h}(s_h,a_h)|\frac{d_{\min}\Delta_{\min}}{2SAH}]]
    \end{align*}
    where the last inequality is because $\Clip[\alpha x| \epsilon] < \alpha \Clip[x|\epsilon]$ as long as $\alpha < 1$. Combining with Lemma \ref{lem:upper_bound_surplus}, we finish the proof.
\end{proof}
Next, we introduce a useful Lemma from \citep{dann2017unifying}:
\begin{lemma}[Lemma 7.4 in \citep{dann2017unifying}]\label{lem:concentration}
    Let $\mathcal{F}_i$ for $i,1...$ be a filtration and $X_1,...X_n$ be a sequence of Bernoulli random variables with $\Pr(X_i=1|\mathcal{F}_{i-1})=P_i$ with $P_i$ being $\mathcal{F}_{i-1}$-measurable and $X_i$ being $\mathcal{F}_i$ measurable. It holds that
    \begin{align*}
        \Pr(\exists n: \sum_{i=1}^n X_i < \sum_{i=1}^n P_i / 2 - W) \leq e^{-W}
    \end{align*}
\end{lemma}

\paragraph{Definition of Good Events}
We first introduce some notations about good events which holds with high probability. 
We override the definition in Cond. \ref{cond:bonus_term} by assigning $\delta=\delta_k = 1/k^\alpha$ at iteration $k$, i.e.:
\begin{align*}
    \cEBk :=& \bigcap_{h\in[H],s_h\in\cS_h,a_h\in\cA_h}\Big\{|\hat P_{k,h}V_{k,h+1}(s_h,a_h) - P_hV_{k,h+1}(s_h,a_h)| < b_{k,h}(s_h,a_h)\} \\
    &\quad\quad\quad\quad\quad\quad\quad\quad\quad\quad\quad\quad \cap \{b_{k,h}(s_h,a_h) \leq B_1\sqrt{\frac{\log (B_2\cdot k^\alpha)}{N_{k,h}(s_h,a_h)}}\Big\}, \\
    \text{with} &\quad b_k=\{b_{k,1},...,b_{k,H}\} \gets \textbf{Bonus}(D_k,1/k^\alpha).
\end{align*}
Besides, we use $\cECk$ to denote the concentration event that 
\begin{align*}
    \cECk:= \bigcap_{h\in[H],s_h\in\cS_h,a_h\in\cA_h} \{N_{k,h}(s_h,a_h) \geq \frac{1}{2}\sum_{\pk=1}^k d^{\piO_\pk}(s_h,a_h) - \alpha \log(SAHk)\}
\end{align*}
Finally, we use $\cEOk$ to denote the good events that the regret of $\algO$ is only at the level $\log k$:
\begin{align*}
    \cEOk := \{\sum_{\tk=1}^\nk V^*_1(s_1)-V^{\piO_\tk}_1(s_1) < C_1 + \alpha C_2 \log \nk\} 
\end{align*}
Based on Cond. \ref{cond:bonus_term}, Cond. \ref{cond:requirement_on_algO} and Lemma \ref{lem:concentration}, we have:
\begin{align*}
    &\Pr(\cEOk) \geq 1-\frac{1}{k^\alpha}, \quad \Pr(\cEBk) \geq 1 - \frac{1}{k^\alpha}\\
    &\Pr(\cECk) \geq 1 - SAH \cdot \exp(-\alpha \log(SAHk)) = 1 - \frac{SAH}{(SAHk)^\alpha} \geq 1-\frac{1}{k^\alpha}
\end{align*}

\begin{restatable}{lemma}{ThmSubOpt}[One Step Sub-optimality Gap Conditioning on Good Events]\label{thm:SubOpt_under_GoodEvents}
    At iteration $k$, on the good events $\cEBk,\cECk$ and $\cEOk$, the sub-optimality gap of $\piE_k$ can be upper bounded by:
    
    (i) when $|\Pi^*|=1$ (i.e. the optimal deterministic policy is unique):
    \begin{align*}
        V_1^*(s_1) - V^{\piE_k}(s_1) \leq 2\EE_{\pi^*}\Big[\sum_{h=1}^H\Clip\Big[\mathbb{I}[k < \btau^{\pi^*}_{s_h,a_h}]\}\cdot H + \mathbb{I}[k \geq \btau^{\pi^*}_{s_h,a_h}]\cdot B_1\sqrt{\frac{8\alpha \log (B_2 k)}{kd^{\pi^*}(s_h,a_h)}}\Big|\epsClip\Big]\Big]
    \end{align*}

    \indent (ii) when $|\Pi^*|>1$ (i.e. there are multiple optimal deterministic policy):
    \begin{align*}
        V_1^*(s_1) - V^{\piE_k}(s_1) \leq 2\EE_{\pi^*_{\text{cover}}}\Big[\sum_{h=1}^H\Clip\Big[\mathbb{I}[k < \btau^{\pi^*_{\text{cover}}}_{s_h,a_h}]\}\cdot H + \mathbb{I}[k \geq \btau^{\pi^*_{\text{cover}}}_{s_h,a_h}]\cdot B_1\sqrt{\frac{4\alpha \log (B_2 k)}{k\tilde{d}^{\pi^*_{\text{cover}}}(s_h,a_h)}}\Big|\epsClip'\Big]\Big]
    \end{align*}
    where $\epsClip:=\frac{\Delta_{\min}}{2H+2}$ and $\epsClip':=\frac{d_{\min}\Delta_{\min}}{2SAH}$; $\pi^*_{\text{cover}}$ and $\tilde{d}^{\pi^*_{\text{cover}}}(s_h,a_h)$ are defined in Thm. \ref{thm:existence_of_well_covered_optpi}; besides,
    \begin{align*}
        \btau^\pi_{s_h,a_h}:=c_\tau\frac{\alpha(C_1+C_2)}{d^{\pi}(s_h,a_h)\Delta_{\min}} \log \frac{\alpha SAH(C_1+C_2)}{d^{\pi}(s_h,a_h) \Delta_{\min}},\quad \btau^{\pi^*_{\text{cover}}}_{s_h,a_h}:=c'_\tau\frac{\alpha(C_1+C_2)}{\tilde{d}^{\pi^*_{\text{cover}}}(s_h,a_h)\Delta_{\min}} \log \frac{\alpha SAH(C_1+C_2)}{\tilde{d}^{\pi^*_{\text{cover}}}(s_h,a_h) \Delta_{\min}}.
    \end{align*}
    for some constant $c_\tau$, $c'_\tau$.
\end{restatable}
\begin{proof}
We first discuss the case when $|\Pi^*|=1$.
\paragraph{Case 1: unique optimal deterministic policy}
As a result of Thm. \ref{thm:clipping_trick}, on the event $\cEBk$, we show that the sub-optimality gap of $\piE_k$ can be upper bounded by:
\begin{align*}
    V^*_1(s_1) - V_1^{\piE_k}(s_1)\leq& 2\EE_{\pi^*}[\sum_{h=1}^H \ddot{\surplus}_{\nk,h}(s_h,a_h)]=\sum_{h=1}^H \sum_{s_h,a_h}d^{\pi^*}(s_h,a_h) \ddot{\surplus}_{\nk,h}(s_h,a_h)
\end{align*}
Because of Lemma \ref{lem:upper_bound_surplus}, the above further implies that:
\begin{align*}
    V^*_1(s_1) - V_1^{\piE_k}(s_1)\leq& 2\EE_{\pi^*}[d\min\{H-h+1, 2B_1\sqrt{\frac{\alpha \log (B_2 k)}{N_{k,h}(s_h,a_h)}}\}].
\end{align*}
Because of Thm. \ref{thm:algO_regret_vs_algP_density}, on the event $\cECk$ and $\cEOk$, we further have:
\begin{align*}
    N_{k,h}(s_h,a_h) \geq& \frac{1}{2}\sum_{\pk=1}^k d^{\piO_\pk}(s_h,a_h) - \alpha \log(SAHk) \\
    \geq& \frac{1}{2}\sum_{\pk=1}^k d^{\pi^*_\pk}(s_h,a_h) - \alpha \log(SAHk) - \frac{1}{\Delta_{\min}}(C_1+\alpha C_2\log k)\\
    \geq& \frac{k}{2}d^{\pi^*}(s_h,a_h) - \alpha \log(SAHk) - \frac{1}{\Delta_{\min}}(C_1+\alpha C_2\log k)
\end{align*}
Now, we define that,
\begin{align*}
    \tau^{\pi^*}_{s_h,a_h} := & \inf_{t:\forall t' \geq t} \{\frac{1}{4}t d^{\pi^*}(s_h,a_h) \geq \alpha \log(SAHt) + \frac{1}{\Delta_{\min}}(C_1+\alpha C_2\log t)\}
\end{align*}
there must exists a constant $c_\tau$ independent with $C_1,C_2,\alpha, d^{\pi^*}(s_h,a_h)$ and $\Delta_{\min}$, such that:
\begin{align*}
    \forall h\in[H],s_h\in\cS_h,a_h\in\cA_h,\quad \tau^{\pi^*}_{s_h,a_h} \leq \btau^{\pi^*}_{s_h,a_h}:=c_\tau\frac{\alpha(C_1+C_2)}{d^{\pi^*}(s_h,a_h)\Delta_{\min}} \log \frac{\alpha SAH(C_1+C_2)}{d^{\pi^*}(s_h,a_h) \Delta_{\min}}.
\end{align*}
Easy to check that, for arbitrary $k \geq \btau^{\pi^*}_{s_h,a_h}$, on the good events, we can verify that $N_{k,h} \geq \frac{k}{4}d^{\pi^*}(s_h,a_h) \geq \frac{\btau^{\pi^*}_{s_h,a_h}}{4}d^{\pi^*}(s_h,a_h) \geq \frac{c_\tau}{4} > 0$, and as a result, we have:
\begin{align*}
    & V^*_1(s_1) - V_1^{\piE_k}(s_1) \\
    \leq& 2\EE_{\pi^*}\Big[\sum_{h=1}^H\Clip\Big[\min\{H-h+1, 2B_1\sqrt{\frac{\alpha \log (B_2 k)}{N_{k,h}(s_h,a_h)}}\}\Big|\epsClip\Big]\Big]\\
    \leq& 2\EE_{\pi^*}\Big[\sum_{h=1}^H\Clip\Big[\mathbb{I}[k < \btau^{\pi^*}_{s_h,a_h}]\}\cdot H + \mathbb{I}[k \geq \btau^{\pi^*}_{s_h,a_h}]\cdot B_1\sqrt{\frac{\alpha \log (B_2 k)}{kd^{\pi^*}(s_h,a_h)/4}}\Big|\epsClip\Big]\Big]
\end{align*}
\paragraph{Case 2: multiple optimal deterministic policies}
The discussion are similar. As a result of Thm. \ref{thm:existence_of_well_covered_optpi}, on the event $\cECk$ and $\cEOk$, we further have:
\begin{align*}
    N_{k,h}(s_h,a_h) \geq& \frac{k}{4}\cdot\max\{\frac{d^*_{h,\min}(s_h,a_h)}{(|\cZ_{h,\text{div}}|+1)H}, d^{\pi^*_{\text{cover}}}(s_h,a_h)\} - \alpha \log(SAHk) - \frac{1}{\Delta_{\min}} \Big(\sum_{k=1}^K V^*_1(s_1)-V^{\pi_k}_1(s_1)\Big) \\
    \geq& \frac{k}{4}\tilde{d}^{\pi^*_{\text{cover}}}(s_h,a_h) - \alpha \log(SAHk) - \frac{1}{\Delta_{\min}}(C_1+\alpha C_2\log k)
\end{align*}
Similarly, we define that,
\begin{align*}
    \tau^{\pi^*_{\text{cover}}}_{s_h,a_h} := & \inf_{t:\forall t' \geq t} \{\frac{t}{8}\tilde{d}^{\pi^*_{\text{cover}}}(s_h,a_h) \geq \alpha \log(SAHt) + \frac{1}{\Delta_{\min}}(C_1+\alpha C_2\log t)\}
\end{align*}
there must exists a constant $c'_\tau$ independent with $C_1,C_2,\alpha, \tilde{d}^{\pi^*_{\text{cover}}}(s_h,a_h)$ and $\Delta_{\min}$, such that:
\begin{align*}
    \forall h\in[H],s_h\in\cS_h,a_h\in\cA_h,\quad \tau^{\pi^*_{\text{cover}}}_{s_h,a_h} \leq \btau^{\pi^*_{\text{cover}}}_{s_h,a_h}:=c'_\tau\frac{\alpha(C_1+C_2)}{\tilde{d}^{\pi^*_{\text{cover}}}(s_h,a_h)\Delta_{\min}} \log \frac{\alpha SAH(C_1+C_2)}{\tilde{d}^{\pi^*_{\text{cover}}}(s_h,a_h) \Delta_{\min}}.
\end{align*}
For arbitrary $k \geq \btau^{\pi^*_{\text{cover}}}_{s_h,a_h}$, on the good events, we can verify that $N_{k,h} \geq \frac{k}{8}d^{\pi^*_{\text{cover}}}(s_h,a_h) > 0$, and as a result, we have:
\begin{align*}
    & V^*_1(s_1) - V_1^{\piE_k}(s_1) \leq 2\EE_{\pi^*_{\text{cover}}}\Big[\sum_{h=1}^H\Clip\Big[\mathbb{I}[k < \btau^{\pi^*_{\text{cover}}}_{s_h,a_h}]\}\cdot H + \mathbb{I}[k \geq \btau^{\pi^*_{\text{cover}}}_{s_h,a_h}]\cdot B_1\sqrt{\frac{8\alpha \log (B_2 k)}{k\tilde{d}^{\pi^*_{\text{cover}}}(s_h,a_h)}}\Big|\epsClip'\Big]\Big]
\end{align*}
\end{proof}
Now, we are ready to prove the main theorem.
\ThmUBPVI*
\begin{proof}
Because the expectation and summation are linear, we have:
\begin{align*}
    \EE[\sum_{k=1}^K V^* - V^{\piE_k}] = \sum_{k=1}^K \EE[V^* - V^{\piE_k}]
\end{align*}
Therefore, in the following, we first provide an upper bound for each $\EE[V^* - V^{\piE_k}]$. Note that the expected regert at step $k$ can be upper bounded by:
\begin{align*}
    &\EE_{\algO, M, \algE}[V_1^*(s_1)-V^{\piE_k}_1(s_1)] \\
    =& \Pr(\cEBk\cap\cECk\cap\cEOk)\EE_{\algO, M, \algE}[V_1^*(s_1)-V^{\piE_k}_1(s_1)|\cEBk\cap\cECk\cap\cEOk]\\
    &+ \Pr(\cEBk^\complement\cup\cECk^\complement\cup\cEOk^\complement)\EE_{\algO, M, \algE}[V_1^*(s_1)-V^{\piE_k}_1(s_1)|\cEBk^\complement\cup\cECk^\complement\cup\cEOk^\complement]\\
    \leq& \Pr(\cEBk\cap\cECk\cap\cEOk)\EE_{\algO, M, \algE}[V_1^*(s_1)-V^{\piE_k}_1(s_1)|\cEBk\cap\cECk\cap\cEOk] + \frac{3H}{k^\alpha}
\end{align*}
Easy to see that $\lim_{K\rightarrow \infty}\sum_{k=1}^K \frac{1}{k^\alpha} < \frac{\alpha}{\alpha - 1} < \infty$ as long as $\alpha > 1$, therefore, in the following, we mainly focus on the first part, and separately discuss its upper bound for the case when $|\Pi^*|=1$ or $|\Pi^*|>1$.
\paragraph{Case 1: $|\Pi^*|=1$ (Unique Optimal Policy)}
We use $\pi^*$ to denote the unique optimal policy and define:
\begin{align*}
    \tau^{\pi^*}_{s_h,a_h,\epsClip} := \inf_{t,\forall t' \geq t} \{B_1\sqrt{\frac{\alpha \log (B_2 t)}{td^{\pi^*}(s_h,a_h)/4}} < \epsClip\}
\end{align*}
Recall that $\epsClip:= \Delta_{\min}/(2H+2)$, it's easy to verify that, there exists a constant $c_{\text{Clip}}$ such that,
\begin{align*}
    \tau^{\pi^*}_{s_h,a_h,\epsClip} \leq \ttau^{\pi^*}_{s_h,a_h,\epsClip} := c_{\text{Clip}} \frac{\alpha H^2}{d^{\pi^*}(s_h,a_h)\Delta^2_{\min}} \log \frac{\alpha B_2 H}{d^{\pi^*}(s_h,a_h)\Delta_{\min}}
\end{align*}
Then, we have:
\begin{align*}
    \lim_{K\rightarrow \infty} \quad &\sum_{k=1}^K 2\EE_{\pi^*}\Big[\sum_{h=1}^H\Clip\Big[\mathbb{I}[k < \btau^{\pi^*}_{s_h,a_h}]\}\cdot H + \mathbb{I}[k \geq \btau^{\pi^*}_{s_h,a_h}]\cdot B_1\sqrt{\frac{\alpha \log (B_2 k)}{kd^{\pi^*}(s_h,a_h)/4}}\Big|\epsClip\Big]\Big]\\
    =&2\EE_{\pi^*}\Big[\sum_{h=1}^H\Big(\sum_{k=1}^{\btau^{\pi^*}_{s_h,a_h}} H+\sum_{k=\btau^{\pi^*}_{s_h,a_h}+1}^K \Clip\Big[B_1\sqrt{\frac{\alpha \log (B_2 k)}{kd^{\pi^*}(s_h,a_h)/4}}\Big|\epsClip\Big]\Big)\Big]\\
    \leq& 2\EE_{\pi^*}[\sum_{h=1}^H\sum_{k=1}^{\btau^{\pi^*}_{s_h,a_h}} H] + 2\EE_{\pi^*}\Big[\sum_{h=1}^H\int_{x=\btau^{\pi^*}_{s_h,a_h}}^{\ttau^{\pi^*}_{s_h,a_h,\epsClip}} B_1\sqrt{\frac{\alpha \log (B_2 x)}{xd^{\pi^*}(s_h,a_h)/4}} dx]\\
    \leq& 2\EE_{\pi^*}[\sum_{h=1}^H\sum_{k=1}^{\btau^{\pi^*}_{s_h,a_h}} H] + 2\sum_{h=1}^H\sum_{\substack{s_h,a_h:\\d^{\pi^*}(s_h,a_h) > 0}} B_1 \sqrt{4\alpha d^{\pi^*}(s_h,a_h)}\int_{x=\btau^{\pi^*}_{s_h,a_h}}^{\ttau^{\pi^*}_{s_h,a_h,\epsClip}} \sqrt{\frac{\log (B_2 x)}{x}}dx
\end{align*}
For the first part, we have:
\begin{align*}
    \EE_{\pi^*}[\sum_{h=1}^H\sum_{k=1}^{\btau^{\pi^*}_{s_h,a_h}} H] =& \sum_{h=1}^H\sum_{\substack{s_h,a_h:\\d^{\pi^*}(s_h,a_h) > 0}} d^{\pi^*}(s_h,a_h) \cdot H \cdot \btau^{\pi^*}_{s_h,a_h} \\
    \leq&c_\tau\frac{\alpha H(C_1+C_2)}{\Delta_{\min}} \cdot \sum_{h=1}^H\sum_{\substack{s_h,a_h:\\d^{\pi^*}(s_h,a_h) > 0}} \log \frac{\alpha SAH(C_1+C_2)}{d^{\pi^*}(s_h,a_h) \Delta_{\min}}
\end{align*}
For the second part, we have:
\begin{align*}
    & \EE_{\pi^*}\Big[\sum_{h=1}^H\int_{x=\btau^{\pi^*}_{s_h,a_h}}^{\ttau^{\pi^*}_{s_h,a_h,\epsClip}} B_1\sqrt{\frac{\alpha \log (B_2 x)}{xd^{\pi^*}(s_h,a_h)/4}} dx]\\
    \leq& \sum_{h=1}^H\sum_{\substack{s_h,a_h:\\d^{\pi^*}(s_h,a_h) > 0}} B_1 \sqrt{4\alpha d^{\pi^*}(s_h,a_h)}\int_{x=\btau^{\pi^*}_{s_h,a_h}}^{\ttau^{\pi^*}_{s_h,a_h,\epsClip}} \sqrt{\frac{\log (B_2 x)}{x}}dx\\
    \leq& 2\sum_{h=1}^H\sum_{\substack{s_h,a_h:\\d^{\pi^*}(s_h,a_h) > 0}} B_1 \sqrt{4\alpha d^{\pi^*}(s_h,a_h)}\cdot 2(\sqrt{\ttau^{\pi^*}_{s_h,a_h,\epsClip}\log B_2\ttau^{\pi^*}_{s_h,a_h,\epsClip}} - \sqrt{\btau^{\pi^*}_{s_h,a_h}\log B_2 \btau^{\pi^*}_{s_h,a_h}}) \tag{Lemma \ref{lem:compute_integral}}\\
    \leq & c_2 \frac{\alpha B_1 H}{\Delta_{\min}} \sum_{h=1}^H\sum_{\substack{s_h,a_h:\\d^{\pi^*}(s_h,a_h) > 0}} \log \frac{\alpha B_2 H}{d^{\pi^*}(s_h,a_h)\Delta_{\min}}
\end{align*}
where in the last step, we drop the term $-\sqrt{\btau^{\pi^*}_{s_h,a_h}\log B_2 \btau^{\pi^*}_{s_h,a_h}}$, and $c_2$ is a constant.

Combining the above results, we have:
\begin{align*}
&\EE_{\algO, M, \algE}[\sum_{k=1}^K V_1^*(s_1)-V^{\piE_k}_1(s_1)]\\
\leq& \sum_{k=1}^K \frac{3H}{k^\alpha} + c_\tau\frac{\alpha(C_1+C_2)}{\Delta_{\min}} \cdot \sum_{h=1}^H\sum_{\substack{s_h,a_h:\\d^{\pi^*}(s_h,a_h) > 0}} \log \frac{\alpha SAH(C_1+C_2)}{d^{\pi^*}(s_h,a_h) \Delta_{\min}} \\
&+ c_2\frac{\alpha B_1 H}{\Delta_{\min}} \sum_{h=1}^H\sum_{\substack{s_h,a_h:\\d^{\pi^*}(s_h,a_h) > 0}} \log \frac{\alpha B_2 H}{d^{\pi^*}(s_h,a_h)\Delta_{\min}} \\
\leq& \frac{3\alpha H}{\alpha - 1} + c_{\algE}\cdot \Big(\sum_{h=1}^H\sum_{\substack{s_h,a_h:\\d^{\pi^*}(s_h,a_h) > 0}} \frac{\alpha(C_1+C_2)}{\Delta_{\min}}\log \frac{\alpha SAH(C_1+C_2)}{d^{\pi^*}(s_h,a_h) \Delta_{\min}} + \frac{\alpha B_1 H}{\Delta_{\min}}\log \frac{\alpha B_2 H}{d^{\pi^*}(s_h,a_h)\Delta_{\min}}\Big)
\end{align*}
where $c_{\algE}$ is some constant.
\paragraph{Case 2: $|\Pi^*|>1$ (Non-Unique Optimal Policy)}
Similar to the discussion above, we define:
\begin{align*}
    \tau^{\pi^*_{\text{cover}}}_{s_h,a_h,\epsClip'} := \inf_{t,\forall t' \geq t} \{B_1\sqrt{\frac{8\alpha \log (B_2 t)}{t\tilde{d}^{\pi^*_{\text{cover}}}(s_h,a_h)}} < \epsClip'\}
\end{align*}
Recall that $\epsClip':= d_{\min}\Delta_{\min}/(2SAH)$, it's easy to verify that, there exists a constant $c_{\text{Clip}}$ such that,
\begin{align*}
    \tau^{\pi^*_{\text{cover}}}_{s_h,a_h,\epsClip'} \leq \ttau^{\pi^*_{\text{cover}}}_{s_h,a_h,\epsClip'} := c_{\text{Clip}} \frac{\alpha (SAH)^2}{\tilde{d}^{\pi^*_{\text{cover}}}(s_h,a_h)(d_{\min}\Delta_{\min})^2} \log \frac{\alpha B_2 SAH}{\tilde{d}^{\pi^*_{\text{cover}}}(s_h,a_h)d_{\min}\Delta_{\min}}
\end{align*}
Following a similar discussion, we have:
\begin{align*}
    \lim_{K\rightarrow \infty} \quad &\sum_{k=1}^K 2\EE_{\pi^*_{\text{cover}}}\Big[\sum_{h=1}^H\Clip\Big[\mathbb{I}[k < \btau^{\pi^*_{\text{cover}}}_{s_h,a_h}]\}\cdot H + \mathbb{I}[k \geq \btau^{\pi^*_{\text{cover}}}_{s_h,a_h}]\cdot B_1\sqrt{\frac{8\alpha \log (B_2 k)}{k\tilde{d}^{\pi^*_{\text{cover}}}(s_h,a_h)}}\Big|\epsClip'\Big]\Big]\\
    \leq& 2\EE_{\pi^*_{\text{cover}}}[\sum_{h=1}^H\sum_{k=1}^{\btau^{\pi^*_{\text{cover}}}_{s_h,a_h}} H] + 2\sum_{h=1}^H\sum_{\substack{s_h,a_h:\\d^{\pi^*_{\text{cover}}}(s_h,a_h) > 0}} B_1 \sqrt{8\alpha d^{\pi^*_{\text{cover}}}(s_h,a_h)}\int_{x=\btau^{\pi^*_{\text{cover}}}_{s_h,a_h}}^{\ttau^{\pi^*_{\text{cover}}}_{s_h,a_h,\epsClip'}} \sqrt{\frac{\log (B_2 x)}{x}}dx\\
    \leq& c'_\tau\frac{\alpha H(C_1+C_2)}{\Delta_{\min}} \cdot \sum_{h=1}^H\sum_{\substack{s_h,a_h:\\d^{\pi^*_{\text{cover}}}(s_h,a_h) > 0}} \log \frac{\alpha SAH(C_1+C_2)}{\tilde{d}^{\pi^*_{\text{cover}}}(s_h,a_h) \Delta_{\min}}\\
    &+c_2' \frac{\alpha B_1 SAH}{d_{\min}\Delta_{\min}} \sum_{h=1}^H\sum_{\substack{s_h,a_h:\\d^{\pi^*_{\text{cover}}}(s_h,a_h) > 0}} \log \frac{\alpha B_2 SAH}{d_{\min}\Delta_{\min}} \tag{Note that $\tilde{d}^{\pi^*_{\text{cover}}} \geq d^{\pi^*_{\text{cover}}}$}
\end{align*} 
Therefore, we have:
\begin{align*}
&\EE_{\algO, M, \algE}[\sum_{k=1}^K V_1^*(s_1)-V^{\piE_k}_1(s_1)]\\
\leq& \frac{3\alpha H}{\alpha - 1} + c'_{\algE}\cdot\Big(\sum_{h=1}^H\sum_{\substack{s_h,a_h:\\d^{\pi^*_{\text{cover}}}(s_h,a_h) > 0}} \frac{\alpha(C_1+C_2)}{\Delta_{\min}}\log \frac{\alpha SAH(C_1+C_2)}{\tilde{d}^{\pi^*_{\text{cover}}}(s_h,a_h)\Delta_{\min}} + \frac{\alpha B_1 SAH}{d_{\min}\Delta_{\min}}\log \frac{\alpha B_2 SAH}{d_{\min}\Delta_{\min}}\Big)
\end{align*}
\end{proof}

\begin{lemma}[Computation of Integral]\label{lem:compute_integral}
    Suppose $p \geq 1$, $b \geq a \geq e/p$, then we have:
    \begin{align*}
        \int_a^b \sqrt{\frac{\log px}{x}} dx \leq 2(\sqrt{b\log pb} - \sqrt{a\log pa})
    \end{align*}
\end{lemma}
\begin{proof}
    \begin{align*}
        \int_a^b \sqrt{\frac{\log px}{x}} dx \leq \int_a^b \frac{1}{\sqrt{x\log px}}+\sqrt{\frac{\log px}{x}} dx= 2\int_a^b (\sqrt{x\log px})' = 2(\sqrt{b\log pb} - \sqrt{a\log pa})
    \end{align*}
\end{proof}

\section{Doubling Trick for $\algO$ Satisfying Cond. \ref{cond:realistic_requirement_on_algO}}\label{appx:Doubling_Trick}
As we briefly mentioned in Sec.\ref{sec:choice_algO}, Cond.\ref{cond:requirement_on_algO} may not holds for some algorithms with near-optimal regret guarantees. 
For example, in \citep{simchowitz2019non,xu2021fine,dann2021beyond}, although these algorithms are anytime, they require a confidence interval $\delta$ as input at the beginning of the algorithm and fix it during the running, which we abstract into the Cond.\ref{cond:realistic_requirement_on_algO} below:
\begin{condition}[Alternative Condition of $\algO$]\label{cond:realistic_requirement_on_algO}
    $\algO$ is an algorithm which returns a deterministic policies $\piO_\tk$ at each iteration $\tk$,
    and for arbitrary fixed $\nk \geq 2$, with probability $1-\delta$, we have the following holds:
    \begin{align*}
        \sum_{\tk=1}^\nk V^*_1(s_1)-V^{\piO_\tk}_1(s_1) \leq C_1 +  C_2 \log \frac{\nk}{\delta}
    \end{align*}
    where $C_1,C_2$ are some parameters depending on $S,A,H$ and $\Delta_h(s_h,a_h)$ and independent with $\nk$.
\end{condition}
As a result, no matter how small $\delta$ is chosen at the beginning, when $\nk \geq \lceil 1/\delta \rceil$, the Cond. \ref{cond:requirement_on_algO} can not be directly guaranteed.
To overcome this issue, we present a new framework in Alg \ref{alg:Tiered_RL_with_Doubling_Trick} inspired by doubling trick. 

\begin{algorithm}[H]
    \textbf{Input}: $\alpha > 1$.\\
    $K_0=1,\quad k = 1,\quad \piE_{1,1}\gets \algE(\{\})$.\\
    \For{$n=1,2,...$}{
        $K_n \gets 2K_{n},\quad \delta_{n-1} = 1/K_n^\alpha,\quad D_{n,1} \gets \{\}$ \\
        \For{$k=1,...,K_{n}$}{
            // Here we do not update $\piE$ \\
            $\piO_{n,k+1}\gets \algO(D_{n,k},\delta_n)$. \\
            $\piE_{n,k+1} =
            \begin{cases}
                \piE_{n-1,K_{n-1}/2+\lceil k / 2\rceil},\quad & \text{If }k \leq K_{n}/2,\\
                \algE(D_{n,k}, 1/k^\alpha),\quad & \text{Otherwise}.\\
            \end{cases}
            $\\
            $\tau_{k+1} \sim \piO_{n,k+1}$\\
            $D_{n,k+1} = D_{n,k} \cup \tau_{n,k+1}$ \\
        }
    }
    \caption{Tiered RL Algorithm with Doubling Trick}\label{alg:Tiered_RL_with_Doubling_Trick}
\end{algorithm}
The basic idea is to iteratively run $\algO$ satisfying Cond.\ref{cond:realistic_requirement_on_algO} from scratch while gradually doubling the number of iterations (i.e. $K_n$) and shrinking the confidence level $\delta_n$ rather than runnning with a fixed $\delta$ forever.
Besides, another crucial part is the computation of $\piE_k$. Instead of continuously updating $\piE$ with the data collected before, we only update the exploitation policy when $k \geq K_n/2$ for each outer loop $n$. As we will discuss in Lemma \ref{lem:conversion_between_algO_condition}, $\algE_{n,k}$ will behave as if the dataset is generated by another online algorithm satisfying Cond. \ref{cond:requirement_on_algO}, and therefore, the analysis based on Cond. \ref{cond:requirement_on_algO} can be adapted here, which we summarize to Thm. \ref{Thm:DoublingTrick} below.
\begin{lemma}\label{lem:conversion_between_algO_condition}
    By running an algorithm satisfying Cond. \ref{cond:realistic_requirement_on_algO} in Alg. \ref{alg:Tiered_RL_with_Doubling_Trick} as $\algO$, for arbitrary $n \geq 1$ and $K_n/2 + 1\leq k < K_n / 2$, we have:
    \begin{align*}
        \Pr(\sum_{k=1}^K V^*-V^{\piO_{n,k}} > C_1' + \alpha C_2' \log k) \leq 1/k^\alpha
    \end{align*}
    with $C_1'=C_1 + (\alpha+1)C_2 \log 2$ and $C_2'=\frac{\alpha+1}{\alpha}C_2$.
\end{lemma}
\begin{proof}
    Based on Cond. \ref{cond:realistic_requirement_on_algO}, we know that:
    \begin{align*}
        \Pr(\sum_{\pk=1}^k V_1^*(s_1)-V_1^{\piO_\pk}(s_1) > C_1 +  C_2 \log \frac{k}{\delta_n}) \leq \delta_n
    \end{align*}
    Since $\delta_n = 1/K_n^\alpha$ and $k \geq K_n / 2 $, we have:
    \begin{align*}
        &\Pr(\sum_{\pk=1}^k V_1^*(s_1)-V_1^{\piO_\pk}(s_1) > C_1 + (1+\alpha) C_2 \log 2k) \\
        =& \Pr(\sum_{\pk=1}^k V_1^*(s_1)-V_1^{\piO_\pk}(s_1) > C_1 + C_2 \log (2k)^{1+\alpha}) \\
        \leq& \Pr(\sum_{\pk=1}^k V_1^*(s_1)-V_1^{\piO_\pk}(s_1) > C_1 + C_2 \log \frac{k}{\delta_n}) \tag{$(2k)^{1+\alpha} \geq 2kK_n^\alpha > k/\delta_n$}\\
        \leq& \delta_n \leq 1/k^\alpha
    \end{align*}
\end{proof}
Now, we are ready to upper bound the regret of $\algE$:
\begin{restatable}{theorem}{ThmRegretDoublingTrick}\label{Thm:DoublingTrick}
    By choosing an arbitrary algorithm satisfying Cond. \ref{cond:realistic_requirement_on_algO} as $\algO$, choosing Alg. \ref{alg:PVI} as $\algE$ and choosing a bonus function satisfying Cond. \ref{cond:bonus_term} as \textbf{Bonus}, the Pseudo regret of $\piE_{n,k}$ in Alg. \ref{alg:Tiered_RL_with_Doubling_Trick} can be upper bounded by:

    (i) $|\Pi^*| = 1$ (unique optimal deterministic policy):
    \begin{align*}
        &\EE[\sum_{n=1}^N \sum_{k=1}^{K_n} V_1^*(s_1) - V^{\piE_{n,k}}]\leq 2H + \frac{9\alpha H}{\alpha - 1} \\
        &+3c_{\algE}\cdot \Big(\sum_{h=1}^H\sum_{\substack{s_h,a_h:\\d^{\pi^*}(s_h,a_h) > 0}} \frac{\alpha(C'_1+C'_2)}{\Delta_{\min}}\log \frac{\alpha SAH(C'_1+C'_2)}{d^{\pi^*}(s_h,a_h) \Delta_{\min}} + \frac{\alpha B_1 H}{\Delta_{\min}}\log \frac{\alpha B_2 H}{d^{\pi^*}(s_h,a_h)\Delta_{\min}}\Big)
    \end{align*}
    \indent (ii) $|\Pi^*| > 1$ (non-unique optimal deterministic policies):
    \begin{align*}
        &\EE[\sum_{n=1}^N \sum_{k=1}^{K_n} V_1^*(s_1) - V^{\piE_{n,k}}]\leq 2H +
        \frac{9\alpha H}{\alpha - 1} \\
        & + 3c'_{\algE}\cdot\Big(\sum_{h=1}^H\sum_{\substack{s_h,a_h:\\d^{\pi^*_{\text{cover}}}(s_h,a_h) > 0}} \frac{\alpha(C'_1+C'_2)}{\Delta_{\min}}\log \frac{\alpha SAH(C'_1+C'_2)}{\tilde{d}^{\pi^*_{\text{cover}}}(s_h,a_h)\Delta_{\min}} + \frac{\alpha B_1 SAH}{d_{\min}\Delta_{\min}}\log \frac{\alpha B_2 SAH}{d_{\min}\Delta_{\min}}\Big)
    \end{align*}
    where $C_1'=C_1 + (\alpha+1)C_2 \log 2$ and $C_2'=\frac{\alpha+1}{\alpha}C_2$.
\end{restatable}
\begin{remark}[$O(\log^2 K)$-Regret of $\algO$]
    Although the regret of $\algE$ stays constant under this framework, it is easy to verify that the pseudo-regret of $\algO$ will be $O(\log^2 K)$ as a result of the doubling trick, which is worse than $O(\log K)$ up to a factor of $\log K$. Therefore, more rigorously speaking, the regret of $\algO$ will be almost near-optimal.
\end{remark}
\begin{proof}
The key observation is that one can decompose the total expected regret into two parts:
\begin{align*}
    \EE[\sum_{n=1}^N \sum_{k=1}^{K_n} V^* - V^{\piE_{n,k}}] =&\EE[\sum_{n=1}^N \sum_{k=1}^{K_n/2} V^* - V^{\piE_{n,k}}] +\EE[\sum_{n=1}^N \sum_{k=K_n/2+1}^{K_n} V^* - V^{\piE_{n,k}}]\\
    =&2\EE[\sum_{n=0}^{N-1} \sum_{k=K_{n}/2+1}^{K_{n}} V^* - V^{\piE_{n,k}}] +\EE[\sum_{n=1}^N \sum_{k=K_n/2+1}^{K_n} V^* - V^{\piE_{n,k}}]\\
    \leq& 2K_0 H +3\EE[\sum_{n=1}^N \sum_{k=K_n/2+1}^{K_n} V^* - V^{\piE_{n,k}}]\numberthis\label{eq:regret_decomposition}
\end{align*}
Therefore, all we need to do is to upper bound the second part of Eq.\eqref{eq:regret_decomposition}.
As a result of Lemma \ref{lem:conversion_between_algO_condition}, we can apply Lemma \ref{thm:SubOpt_under_GoodEvents} to upper bound the regret of the policy sequence $\{\{\piE_{n,k}\}_{n=1}^N\}_{k=K_n/2+1}^{K_n}$, since the Cond. \ref{cond:requirement_on_algO} is satisfied when generating those policies. Therefore, the Pseudo regret $\EE[\sum_{n=1}^N \sum_{k=K_n/2+1}^{K_n} V^* - V^{\piE_{n,k}}]$ can be upper bounded by extending the results in Thm. \ref{thm:UB_Regret_PVI} here, and we finish the proof.

\end{proof}

\section{Experiments}\label{appx:experiments}
\subsection{Experiment Setup}
\paragraph{Environment}
We test our algorithms in tabular MDPs with randomly generated transition and rewards functions. To generate the MDP, for each layer $h$ and each state action pair $(s_h,a_h)$, we first sample a random vector $\mP(\cdot|s_h,a_h)$, where each element is uniformly sampled from $\{1,2,3...,10\}$, and then normalize it to a valid probability vector. Besides, the reward function is set to $\xi/10$ where $\xi$ is randomly generated from $\{1,2,...,10\}$ to make sure it locates in $[0,1]$.

\paragraph{Algorithm}
We implement the StrongEuler algorithm in \citep{simchowitz2019non} as $\algO$ and construct the same adaptive bonus term (Alg. 3 in \citep{simchowitz2019non}) for $\algE$ to match Cond. \ref{cond:bonus_term}.
Although for the convenience of analysis, in our Framework \ref{alg:general_learning_framework}, we do not consider to use the data generated by $\algE$, in experiments, we use both $\tauO$ and $\tauE$, which slightly improves the performance.
Besides, in practice, we observe that the bonus term is quite loose, and it will take a long time before the estimated $Q/V$ value fallen in the interval $[0, H]$, which is the value range of true value functions.
Therefore, we introduce a multiplicator $\alpha$ and adjust the bonus term from $b_{k,h}$ to $\alpha \cdot b_{k,h}$, and set $\alpha = 0.25$ for both $\algO$ and $\algE$.

\subsection{Results}
We test the algorithms in tabular MDPs with $S=A=H=5$\footnote{The code can be find in \url{https://github.com/jiaweihhuang/Tiered-RL-Experiments}.}. Although the minimal gap $\Delta_{\min}$ is hard to control since we generate the MDP in a random way, we filter out three random seeds in MDP construction, which correspond to minimal gaps (approximately) equal to 0.0015, 0.003 and 0.009, respectively. We report the accumulative regret in Fig. \ref{fig:simulation_results}.

As predicted by our theory, $\algE$ can indeed achieve constant regret in contrast with the continuously increasing regret of $\algO$, which demonstrates the advantage of leveraging tiered structure.
\begin{figure}
    \centering
    \includegraphics[scale=0.45]{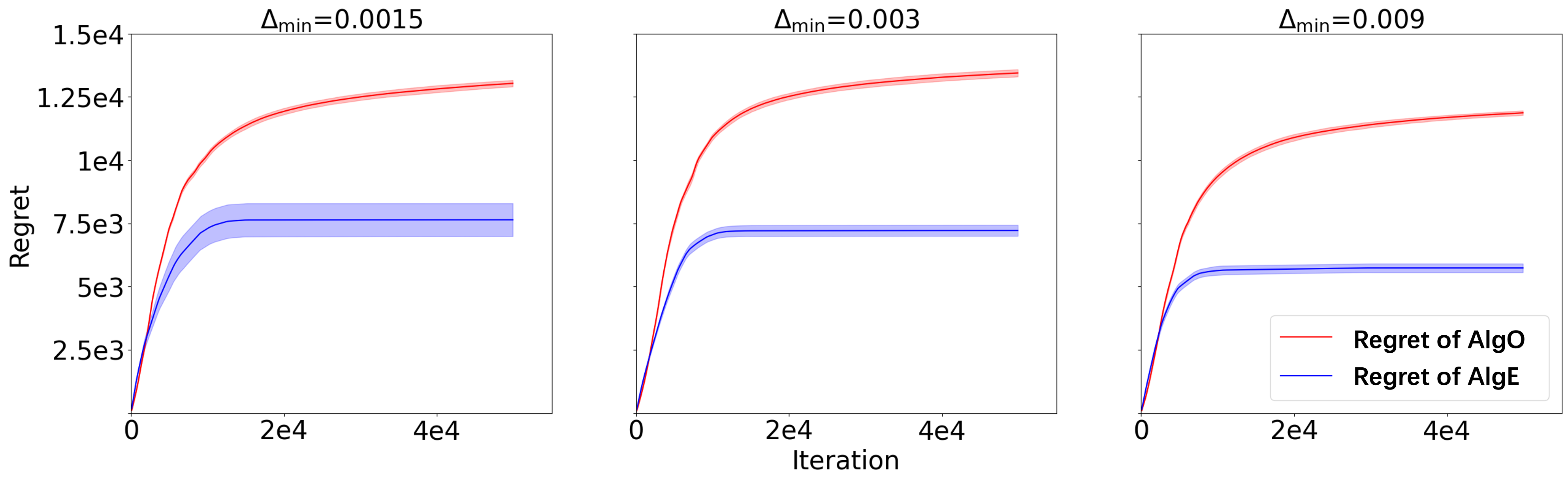}
    \caption{Simulation results with $S=A=H=5$ and different $\Delta_{\min}$, averaged over 10 different random seeds. Error bars show double the standard errors, which correspond to 95\% confidence intervals. Our choice of $\algE$ can achieve constant regret as predicted by theory. We can also see the tendency that larger $\Delta_{\min}$ will result in smaller accumulative regret.}
    \label{fig:simulation_results}
\end{figure}


\end{document}